\documentclass[letterpaper, 10 pt, journal, twoside]{IEEEtran}  

\IEEEoverridecommandlockouts                              

\usepackage{amsmath,amssymb,amsfonts}
\usepackage{graphicx}
\usepackage{algorithm}
\usepackage{algpseudocode}
\usepackage{subfig}
\usepackage{overpic}
\usepackage[colorlinks=true, allcolors=blue]{hyperref}
\usepackage{multirow}
\usepackage{threeparttable}
\usepackage{bm}
\usepackage{cite}

\newtheorem{Sup}{\bf Assumption}

\newtheorem{theorem}{\bf Theorem}

\newcommand{\bxx}{\mathbf{x}}
\newcommand{\bx}{\mathbf{X}}
\newcommand{\bX}{\mathbf{X}}

\newcommand{\bY}{\mathbf{Y}}
\newcommand{\bD}{\mathbf{D}}
\newcommand{\bC}{\mathbf{C}}
\newcommand{\bQ}{\mathbf{Q}}
\newcommand{\bS}{\mathbf{S}}
\newcommand{\bs}{\mathbf{s}}
\newcommand{\be}{\mathbf{e}}
\newcommand{\bq}{\mathbf{q}}
\newcommand{\bc}{\mathbf{c}}
\newcommand{\balpha}{\bm{\alpha}}
\newcommand{\bxi}{\bm{\xi}}
\newcommand{\br}{\mathbf{r}}
\newcommand{\mP}{\mathcal{P}}
\newcommand{\mX}{\mathcal{X}}
\newcommand{\mY}{\mathcal{Y}}
\newcommand{\mN}{\mathcal{N}}

\newcommand{\bmu}{\bm{\mu}}
\newcommand{\bSigma}{\bm{\Sigma}}

\newcommand{\red}{\textcolor{black}}
\newcommand{\blue}{\textcolor{black}}

\begin{document}

\title{\LARGE \bf
Resource-Efficient Cooperative Online Scalar Field Mapping via Distributed Sparse Gaussian Process Regression
}

\author{Tianyi Ding, Ronghao Zheng$^{\dag}$, Senlin Zhang, Meiqin Liu 
\thanks{Manuscript received: September 28, 2023; Revised: December 7, 2023; Accepted: January 8, 2024.}
\thanks{This paper was recommended for publication by
Editor Javier Civera upon evaluation of the Associate Editor and Reviewers’
comments. }
 \thanks{The authors Tianyi Ding, Ronghao Zheng, Senlin Zhang and Meiqin Liu are with the College of Electrical Engineering, Zhejiang University, Hangzhou 310027,  China. Senlin Zhang is also with the Jinhua Institute of Zhejiang University, Jinhua 321036, China. Meiqin Liu is also with the National Key Laboratory of Human-Machine Hybrid Augmented Intelligence, Xi’an Jiaotong University, Xi’an 710049, China. All authors are also with the National Key Laboratory of Industrial Control Technology, Zhejiang University, Hangzhou 310027, China. Emails: {\tt\small \{ty\_ding, rzheng, slzhang, liumeiqin\}@zju.edu.cn}}
	\thanks{$^{\dag}$Corresponding author}
	\thanks{This work was supported by the National Natural Science Foundation of China under Grant 62173294, the Zhejiang Provincial Natural Science Foundation of China under Grant LZ24F030001, and the National Natural Science Foundation of China under Grant U23B2060.}
 \thanks{Digital Object Identifier (DOI): see top of this page.}}

\markboth{IEEE Robotics and Automation Letters. Preprint Version. Accepted January, 2024}
{Ding \MakeLowercase{\textit{et al.}}: Resource-Efficient Cooperative Online Scalar Field Mapping via Distributed Sparse Gaussian Process Regression}

\maketitle

\begin{abstract}
	Cooperative online scalar field mapping is an important task for multi-robot systems. Gaussian process regression is widely used to construct a map that represents spatial information with confidence intervals. However, it is difficult to handle cooperative online mapping tasks because of its high computation and communication costs. This letter proposes a resource-efficient cooperative online field mapping method via distributed sparse Gaussian process regression. A novel distributed online Gaussian process evaluation method is developed such that robots can cooperatively evaluate and find observations of sufficient global utility to reduce computation. The error bounds of distributed aggregation results are guaranteed theoretically, and the performances of the proposed algorithms are validated by real online light field mapping experiments.
\end{abstract}

\section{INTRODUCTION}


\red{Online mapping in unknown environments is a critical mission in environmental monitoring and exploration applications. \blue{Compared with static sensor networks \cite{singh_model_2010}, \cite{han_coverage_2018}, multi-robot systems have higher exploration efficiency and flexibility for large-scale cooperative online mapping \cite{10038280}, \cite{newaz_decentralized_2023}.} Each robot moves and measures data locally in fields and then shares information with neighbors to construct the global scalar field map. }

\red{Gaussian process regression (GPR) is a framework for nonlinear non-parametric Bayesian inference widely used in field mapping \cite{9340836}, \cite{9981548}, \cite{jin_adaptive-resolution_2022}. Different from the traditional discrete occupancy grid model \cite{lan_learn_2017}, \cite{hornung_octomap_2013}, GPR predicts a continuous function to learn the dependencies between points in the fields and represent the map. Another benefit of GPR mapping is the explicit uncertainty representation, which can be used for efficient path planning. }

Unfortunately, the standard GPR requires $\mathcal{O}(N^3)$ time complexity with the training sample size $N$, which is not suitable for large-scale online mapping due to the high cost in computation, communication, and memory \cite{liu_when_2020}. In this letter, we consider cooperative online scalar field mapping for distributed multi-robot systems and propose a resource-efficient distributed sparse Gaussian process regression method to solve these challenges.

\begin{figure}
	\centering
	\includegraphics[width=5.5cm,height=3.cm,trim=0 0 0 0]{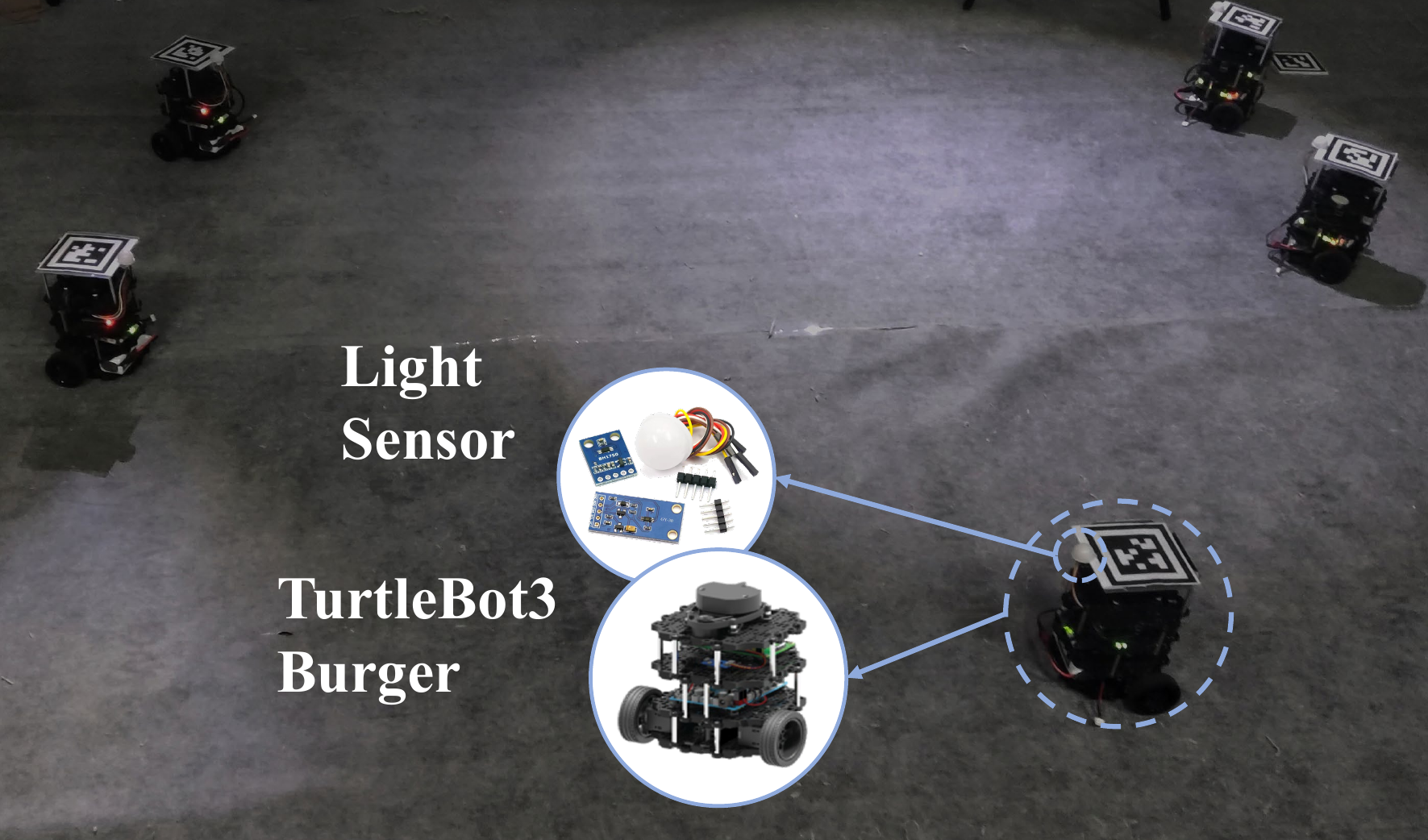}	
        \includegraphics[width=5.5cm,height=3.3cm,trim=0.2cm 0 0.15cm 0]{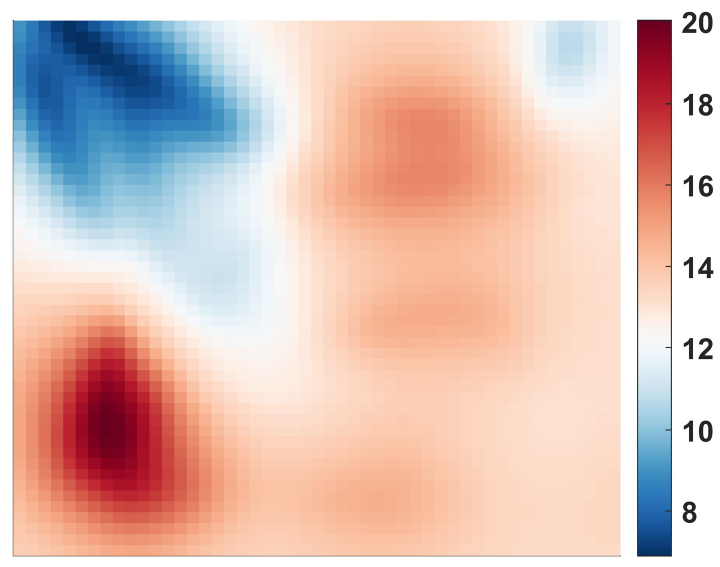}
	\caption{\blue{Light field mapping experiments in a $7.5~\mathrm{m} \times 5~\mathrm{m}$ workspace (top)} and distributed light field mapping results (bottom). The value is the light intensity (lx).}
	\label{fig1}
\end{figure}

\subsection{Related Work}

To speed up GPR, sparse variational GPR (SVGPR) is developed \cite{titsias_variational_2009}, \cite{burt_rates_2019}, which employs a set of $M \ll N$ pseudo-points based on KL divergence to summarize the $N$ sample data thereby reducing computational costs to $\mathcal{O}(NM^2)$. \red{Meanwhile, many Gaussian process aggregation methods are proposed for parallel computation, like product-of-experts (PoE) \cite{pmlr-v37-deisenroth15} and robust Bayesian committee machine \cite{pmlr-v80-liu18a}.} Then, SVGPR combined parallel computation \cite{pmlr-v48-hoang16} and stochastic optimization \cite{pmlr-v37-hoang15} methods are developed to further reduce computational costs. Moreover, Refs. \cite{csato_sparse_2002} and \cite{bui_streaming_2017} provide recursive frameworks for standard GPR and sparse variational GPR to handle online streaming data. Ref. \cite{wilcox_solar-gp_2020} combine these sparse variational, mini-batches and streaming methods and apply them to Bayesian robot model learning.
\red{However, these methods both require a central node, in which SVGPR \cite{titsias_variational_2009,burt_rates_2019,csato_sparse_2002,bui_streaming_2017,wilcox_solar-gp_2020} requires global information of the observations for sparse approximations, and the parallel methods \cite{pmlr-v37-deisenroth15,pmlr-v80-liu18a,pmlr-v48-hoang16,pmlr-v37-hoang15} require a central node for model aggregations.} Therefore, these methods cannot be easily extended to distributed multi-robot systems since all observations and models are required to communicate across networks, which increases inter-robot communications and causes repeated computations about the same observations on different robots.

For dynamic networks without a center, distributed Gaussian process aggregation methods based on average consensus are proposed in \cite{yuan_communication-aware_2020}, \cite{lederer_cooperative_2023} and \cite{9561566}. Then, Ref. \cite{jang_multi-robot_2020} applies an online distributed Gaussian process method to scalar field mapping. These methods avoid communicating observations on networks. However, these methods do not provide error bounds with centralized aggregation methods and do not consider the sparse computation, causing $\mathcal{O}(N^3)$ computational costs. Ref. \cite{kepler_approach_2020} provides a local online Gaussian process evaluation method to reduce computation and inter-robot communications by only computing and communicating partial observations. Ref. \cite{zobeidi_dense_2022} employs the SVGPR for multi-robot systems, in which each robot obtains sparse subsets based on its local observations. \red{Unfortunately, these approaches have limited global accuracy because the online evaluation only uses local information.}

\subsection{Main Contributions}
This letter proposes a distributed sparse online Gaussian process regression method for cooperative online scalar field mapping. The main contributions of this letter are as follows.
\begin{itemize}
	\item Propose a distributed recursive Gaussian process based on dynamic average consensus for streaming data and cooperative online mapping. The error bounds between distributed consensus results and the centralized results are guaranteed theoretically (Theorem~1). 
	\item Propose a novel distributed online evaluation method for Gaussian process. \red{Compared with \cite{kepler_approach_2020} and \cite{zobeidi_dense_2022}, each robot online evaluates observations based on distributed consensus results instead of only local information to obtain better global accuracy without additional communications.}
	\item Develop a resource-efficient cooperative online scalar field mapping method via
	distributed sparse Gaussian process regression. The performances of the proposed algorithms are evaluated by real online light field mapping experiments.
\end{itemize}

\section{Preliminaries And System Modeling}


\subsection{Problem Statement}

\textbf{(Network Model)} Consider a team of $p$ robots, labeled by $i \in V = \{1, . . . , p\}$. The communication networks at time $t$ can be represented by a directed graph $\mathcal{G}(t) = (V , E(t))$ with an edge set $E(t) \subset V \times V$. We consider that $(i, j) \in E(t)$ if and only if node $i$ communicates to node $j$ at time $t$. Define the set of the neighbors of robot $i$ at time $t$ as $\mathbb{N}_{i,t} =\{j\in V|(i,j)\in E(t)\}$. The matrix $A(t):=[a_{ij,t}]_{i,j=1}^{p,p} \in \mathbb{R}^{p\times p}$ represents the adjacency matrix of $\mathcal{G}(t)$.

\textbf{(Observation Model)} Each robot $i$ independently senses the fields $f:\mX \rightarrow \mY$ and gets the measurements $y_{i,t} \in \mY$ with zero-mean Gaussian noise $e_{i,t}$ at time $t$, where $\mX \subset \mathbb{R}^{n_\bxx}$ is the input feature space for $f$ and $\mY \subset \mathbb{R}$ is the corresponding output space. \red{This letter considers 2D cases ($n_\bxx=2$) and it can also be used in 3D cases as \cite{zobeidi_dense_2022}}. The observation model is given by
\begin{equation*}
	y_{i,t} = f\left(\bxx_{i,t}\right) + e_{i,t},~e_{i,t}\sim\mN(0,\sigma_e^2),
\end{equation*}
where $\bxx_{i,t} \in \mX$ is the position of robot $i$ at time $t$. Let the local datasets of robot $i$ at time $t$ be in the form $\bD_{i,t}:=\{\bX_{i,t},\bY_{i,t}\}$, where $\bX_{i,t} = \{\bxx_{i,k}\}_{k=1}^t,~\bY_{i,t} = \{y_{i,k}\}_{k=1}^t$ are the sets of observations from begin to current time $t$.


\subsection{Gaussian Process Regression}


Let the unknown fields $f:\mX \rightarrow \mY$ be the target regression functions.
A robot can use Gaussian process regression \cite{GPForML} to predict the latent function value $f(\bxx^*)$ at the test data points $\bxx^* \in \mX$ using the training datasets $\bD = \{\bX, \bY \}$. 

In GPR, the Gram function $K(\bx,\bx):=[\kappa(\bxx_i,\bxx_j)]_{i,j=1}^{N,N}$ is constructed from the kernel function $\kappa:\mX \times \mX \rightarrow \mathbb{R}$. \red{A commonly used kernel function is the squared exponential (SE) kernel defined as}
\begin{equation} \label{eq:kernel}
	\kappa(\bxx_i,\bxx_j) = \sigma_f^2 \exp\left(-\frac{\Vert \bxx_i - \bxx_j\Vert^2}{\bSigma_\eta^2}\right)
\end{equation}
with hyper-parameters $\{\sigma_f,\bSigma_\eta\}$. Then, the posterior distribution is derived as $\rho(\bxx^*):=\mP(f(\bxx^*)|\bX,\bY,\bxx^*) \sim \mN \left(\bmu(\bxx^*),\bSigma(\bxx^*)\right)$, where
\begin{equation} \label{LGPR}
	\begin{aligned}
		\bmu(\bxx^*) =  &K(\bxx^*,\bx) [K(\bx,\bx)+\sigma_e^2\mathbf{I}]^{-1}\bY, \\
		\bSigma(\bxx^*) = &K(\bxx^*,\bxx^*)\\ &-K(\bxx^*,\bx)[K(\bx,\bx)+\sigma_e^2\mathbf{I}]^{-1}K(\bx,\bxx^*).
	\end{aligned}
\end{equation}
In this letter, we assume the hyper-parameters can be typically selected by maximizing the marginal log-likelihood offline \cite{GPForML} and focus on GP predictions \eqref{LGPR} for field mapping, like Refs. \cite{yuan_communication-aware_2020}, \cite{lederer_cooperative_2023}, \cite{jang_multi-robot_2020}, \cite{kepler_approach_2020} and \cite{zobeidi_dense_2022}.

\subsection{Recursive Gaussian process update for streaming data}
In the online mapping mission, the data arrives sequentially. At any time $t$, the new data point $\{\bxx_{t},y_{t}\}$ is added to datasets $\bD_{t} = [\bD_{t-1}; \bxx_{t},y_{t}]$ in streaming set. \red{In this case, the Gaussian process prediction \eqref{LGPR} will suffer from slow updating.} To speed up the online mapping, we employ the recursive Gaussian process update \cite{csato_sparse_2002},
\begin{equation} \label{RLGPR}
	\begin{aligned}
		\bmu(\bxx^*) &=  K(\bxx^*,\bx_t)\balpha_t, \\
		\bSigma(\bxx^*) &= K(\bxx^*,\bxx^*)+K(\bxx^*,\bx_t)\bC_t K(\bx_t,\bxx^*),
	\end{aligned}
\end{equation}
where $\balpha_t \in \mathbb{R}^{t}$ and $\bC_t \in \mathbb{R}^{t \times t}$ are recursive updated variables. Suppose a new data point $\{\bxx_{t+1},y_{t+1}\}$ has been collected, the recursive updates are designed as,
\begin{equation} \label{re_update}
\begin{aligned}
	&\balpha_{t+1} = \begin{bmatrix}	\balpha_t \\ 0	\end{bmatrix} + q_{t+1}\bs_{t+1}, \\
	&\bC_{t+1} = \begin{bmatrix}	\bC_t & \mathbf{0}_t \\ \mathbf{0}^T_t & 0	\end{bmatrix} + r_{t+1} \bs_{t+1} \bs^T_{t+1}, \\
	&\bQ_{t+1} = \begin{bmatrix}	\bQ_t & \mathbf{0}_t \\ \mathbf{0}^T_t & 0	\end{bmatrix} + \gamma_{t+1} \be_{t+1} \be^T_{t+1}, \\
	&\bs_{t+1} = \begin{bmatrix}	\bC_tK(\bx_t,\bxx_{t+1}) \\ 1	\end{bmatrix},~
	\be_{t+1} = \begin{bmatrix}	\bQ_tK(\bx_t,\bxx_{t+1}) \\ -1	\end{bmatrix}
\end{aligned}
\end{equation}
with the scalars $q_{t+1}$, $r_{t+1}$ and $\gamma_{t+1}$ given by
\begin{equation*}
    q_{t+1} =  \frac{y_{t+1}- K(\bxx_{t+1},\bx_{t}) \balpha_t}{\sigma_e^2 + \kappa(\bxx_{t+1},\bxx_{t+1}) + K(\bxx_{t+1},\bx_{t}) \bC_t K(\bx_{t},\bxx_{t+1})}, 
\end{equation*}
\begin{equation*}
    r_{t+1} =  \frac{-1}{\sigma_e^2 + \kappa(\bxx_{t+1},\bxx_{t+1}) + K(\bxx_{t+1},\bx_{t}) \bC_t K(\bx_{t},\bxx_{t+1})},
\end{equation*}
\begin{equation} \label{eq:q,r}
    \gamma_{t+1} =  \frac{1}{\kappa(\bxx_{t+1},\bxx_{t+1}) - K(\bxx_{t+1},\bx_{t}) \bQ_t K(\bx_{t},\bxx_{t+1})}.
\end{equation}
Upon comparison with \eqref{RLGPR}, $1/\gamma_{t+1}$ can be interpreted as the recursive predictive variance of the new data point $\bxx_{t+1}$ given the observations at $\bD_t$ without noise. This property and the variable $\bQ$ will be used in Section~IV for sparse approximation.

\section{Consensus-based Distributed Gaussian Process Regression}
In this section, we introduce the distributed Gaussian process regression methods for cooperative online field mapping. Each robot constructs an individual field map and aggregates the global map in a distributed fashion using dynamic average consensus methods. 


\subsection{Consensus-based distributed Gaussian posterior update}

At time $t$, each robot $i$ first samples a new training pair $\{\bxx_{i,t},y_{i,t}\}$ in the fields. \red{Then, robots recursive update local Gaussian process posterior mapping $\rho_{i,t}^{[L]}(\bxx^*)$ by \eqref{RLGPR} and \eqref{re_update}, where $\bxx^*$ are the test map grid positions.} With some abuse of notations, denote $\rho_{i,t}^{[L]}$ and $\rho_{i,t}^{[D]}$ as the local and distributed Gaussian posterior mapping of robot $i$ at time $t$, respectively.

\red{Next, we employ discrete-time dynamic average consensus methods \cite{zhu_discrete-time_2010} for distributed map aggregation.} The first-order iterative aggregation process can be written as
\begin{equation} \label{eq:DGPR}
	\begin{aligned}
		\bxi_{i,t+1} &= \bxi_{i,t} + \sum_{j \in \mathbb{N}_{i,t}} a_{ij,t} \left(\bxi_{j,t} - \bxi_{i,t} \right) + \Delta \br_{i,t}, \\
		\bxi_{i,0} &= \br_{i,0},~ \Delta \br_{i,0} = \mathbf{0},
	\end{aligned}
\end{equation} 
where $\bxi_{i,t}$ denotes the local consensus state of robot $i$ at time $t$ and $\Delta \br_{i,t} = \br_{i,t} - \br_{i,t-1}$ is the reference input, which determines the converged consensus results. \red{Many aggregation schemes can be used to design this reference input. We employ the PoE method in this letter.} The distributed Gaussian process update can be designed as
\begin{equation} \label{dist_update}
	\br_{i,t} = \begin{bmatrix}
		\bmu_{i,t}^{[L]} \cdot  (\hat{\bSigma}_{i,t}^{[L]})^{-1} \\  (\hat{\bSigma}_{i,t}^{[L]})^{-1}
	\end{bmatrix},~
	\begin{bmatrix}
		\bmu_{i,t}^{[D]} \\  \bSigma_{i,t}^{[D]}
	\end{bmatrix} = 
	\begin{bmatrix}
		\bxi_{i,t+1}^{[1]} \cdot ({\bxi_{i,t+1}^{[2]}})^{-1} \\  ({\bxi_{i,t+1}^{[2]}})^{-1}
	\end{bmatrix},
\end{equation}
where $\bxi_{i,t+1}^{[k]}$ denotes the $k$-th element of the local consensus state and $\hat{\bSigma}_{i,t}^{[L]}:= \bSigma_{i,t}^{[L]} + \sigma_n^2$ is the local Gaussian posterior variance with the correction biases $\sigma_n^2$ for avoiding singularity. Due to the theoretical foundation of consensus algorithms, the distributed aggregation mean $\bmu_{i,t}^{[D]}$ will converge to the centralized PoE results
\begin{equation} \label{agg_result}
	\begin{aligned}
		\tilde{\bmu} = \sum_{i=1}^p \frac{ (\hat{\bSigma}_{i,t}^{[L]})^{-1} } {\sum_i^p (\hat{\bSigma}_{i,t}^{[L]})^{-1} } \bmu_{i,t}^{[L]},
	\end{aligned}
\end{equation} 
which will be discussed latter. In particular, robots are only required to communicate the local consensus state $\bxi_{i,t}$ with constant size, instead of the online datasets $\bD_{i,t}$ with growing size.

\subsection{Error bounds derivation for distributed fusion}
For the derivation of distributed aggregation error bounds, we introduce the following assumptions.

\begin{Sup} \label{Sup:PSC}
	(Periodical Strong Connectivity). There is some positive integer $B \geq 1$ such that, for all time instant $t \geq 1$, the directed graph $(V,E(t) \cup E(t+1) \cup ... \cup E(t+B-1))$ is strongly connected.
\end{Sup}
\begin{Sup} \label{Sup:Non-degeneracy}
	(Non-degeneracy and Doubly Stochastic). There exists a constant $\varphi>0$ such that $a_{ii,t} = 1 - \sum_{j \in \mathbb{N}_{i,t}} a_{ij,t} \geq \varphi$, and $a_{ij,t}~(j \in \mathbb{N}_{i,t})$ satisfies $a_{ij,t} \in \{0\} \cup [\varphi, 1],~\forall t \geq 1$. And there hold that $\mathbf{1}^TA(t) = \mathbf{1}^T$ and $A(t)\mathbf{1} = \mathbf{1},~\forall t \geq 1$.
\end{Sup}

The Assumption~\ref{Sup:PSC} describes the periodical strong connectivity between robots and Assumption~\ref{Sup:Non-degeneracy} defines the parameters selection of adjacency matrix, which are common for average consensus methods \cite{zhu_discrete-time_2010}.

\begin{Sup} \label{Sup:obs_bound}
	(Bounded Observations). There exists time invariant constants $\bar{y}$ and $\bar{\mu}$ such that all observations $y_{i,t}$ and local Gaussian process predictions $\bmu^{[L]}_{i,t}(\bxx^*)$ satisfy $\vert y_{i,t} \vert \leq \bar{y}$,~$\vert \bmu^{[L]}_{i,t}(\bxx^*) \vert \leq \bar{\mu}$, $~\forall t \geq 1,~\forall i \in V,~\forall \bxx^* \in \mX$.
\end{Sup}


The average consensus methods generally suppose the derivative of the consensus state is bounded \cite{yuan_communication-aware_2020}, \cite{lederer_cooperative_2023} and \cite{zhu_discrete-time_2010}. The derivative bounds are related to the consensus errors and can be used for parameter selections of the communication network in practice, like communication periods and adjacency matrix. \red{However, the derivative bound is hard to get by sample if the field distribution is unknown in practice.} This letter employs the character of the recursive Gaussian process to loosen this requirement to the bounded observations (Assumptions~\ref{Sup:obs_bound}) for easier parameter selections on online mapping. Moreover, since the Gaussian process predictions are determined by the observations and the prior kernel function, it is bounded as the observations.

\begin{theorem} \label{error bound}
	Suppose Assumptions \ref{Sup:PSC}~$\sim$~\ref{Sup:obs_bound} hold and each robot $i$ runs recursive Gaussian process \eqref{RLGPR} and first-order consensus algorithms \eqref{eq:DGPR}. Let $\hat{\delta}_1$,~$\hat{\delta}_2$ and $\eta$ be positive constants, where
	\begin{equation}
		\begin{aligned}
			&\hat{\delta}_1 = \frac{\bar{y} \sigma_f^2(\sigma_n^2+\sigma_f^2)} {\sigma_n^2(\sigma_e^2+\sigma_f^2)} + \frac{\bar{\mu} \sigma_f^4}{\sigma_n^2(\sigma_e^2+\sigma_f^2)}, \\
			&\hat{\delta}_2 = \frac{ \sigma_f^4}{\sigma_n^2(\sigma_e^2+\sigma_f^2)},\\
			&\eta = \frac{4(pB-1)}{\varphi^{0.5p(p+1)B-1}}.
		\end{aligned}
	\end{equation}
	Select the correction term $\sigma_n^2$ to satisfy $\sigma_n^2 \geq \frac{\eta \sigma_f^4}{\sigma_e^2+\sigma_f^2}$. Then, the error bounds between $\bmu_{i,t}^{[D]}(\bxx^*)$ and PoE results $\tilde{\bmu}(\bxx^*)$ \eqref{agg_result} at any test point $\bxx^*$ will converge to
	\begin{equation}
		\lim\limits_{t\rightarrow \infty} \left\lvert \bmu_{i,t}^{[D]}(\bxx^*) - \tilde{\bmu}(\bxx^*)  \right\rvert \leq  \alpha +\beta \tilde{\bmu}(\bxx^*),~\forall i \in V,~ \bxx^* \in \mX,
	\end{equation}
	where 
	\begin{equation}
		\alpha = \frac{\eta \hat{\delta}_1 }{1 + \eta \hat{\delta}_2 }, ~
		\beta = \left\lvert \frac{ \eta \hat{\delta}_2 }{1 - \eta \hat{\delta}_2 } \right\rvert.
	\end{equation}
	Furthermore, if there is a constant $h>0$, $\Delta\br_{i,t+1} = \mathbf{0}$ for any time $t>h$ holds, the error bounds will converge to zero, i.e., $\bmu_{i,t}^{[D]}(\bxx^*) \rightarrow \tilde{\bmu}(\bxx^*),~\forall i \in V,~ \bxx^* \in \mX$ as time $t \rightarrow \infty$.
\end{theorem}

\begin{proof}
	The proof sketch is in the Appendix.
\end{proof}

\blue{The error bounds given in Theorem~1 do not require the specific GP kernels and hyper-parameters and thus can be used in other distributed GP systems for communication parameter selections.}

\section{Distributed sparse Gaussian process approximation}
To speed up the GPR for online mapping, it is necessary to design Gaussian process sparse approximation methods. In this section, 
a novel distributed metric is proposed for data evaluations such that predictions have better global performance without additional communication costs.

\begin{figure*}
	\centering
	\subfloat[Centralized Compression]{\includegraphics[width=5.6cm,height=4cm,trim=1.3cm 0 1cm 0.7cm, clip]{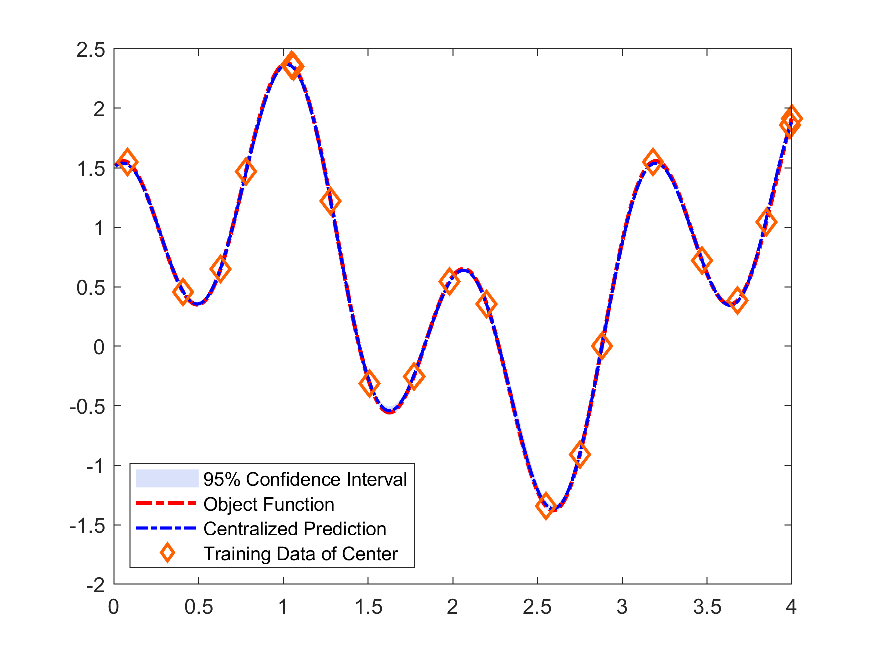}
	}
	\subfloat[Local Compression \cite{kepler_approach_2020}]{
		\begin{overpic}[width=5.6cm,height=4cm,trim=1.3cm 0 1cm 0.7cm, clip]{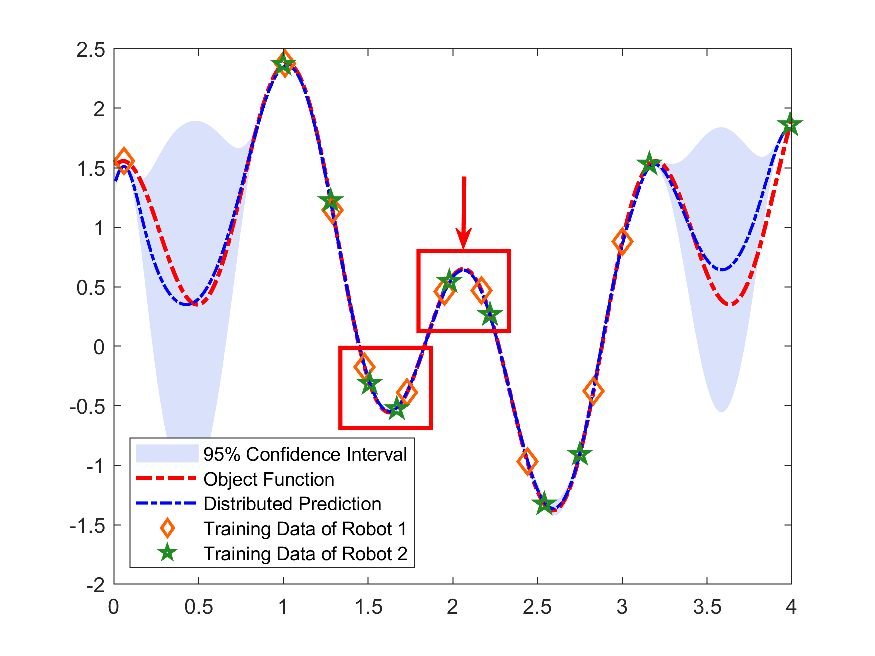}
			\put(34,60){\small \color[rgb]{0,0,0}{Repeated Information}} 
		\end{overpic}
	}
	\subfloat[Distributed Compression]{\includegraphics[width=5.6cm,height=4cm,trim=1.3cm 0 1cm 0.7cm, clip]{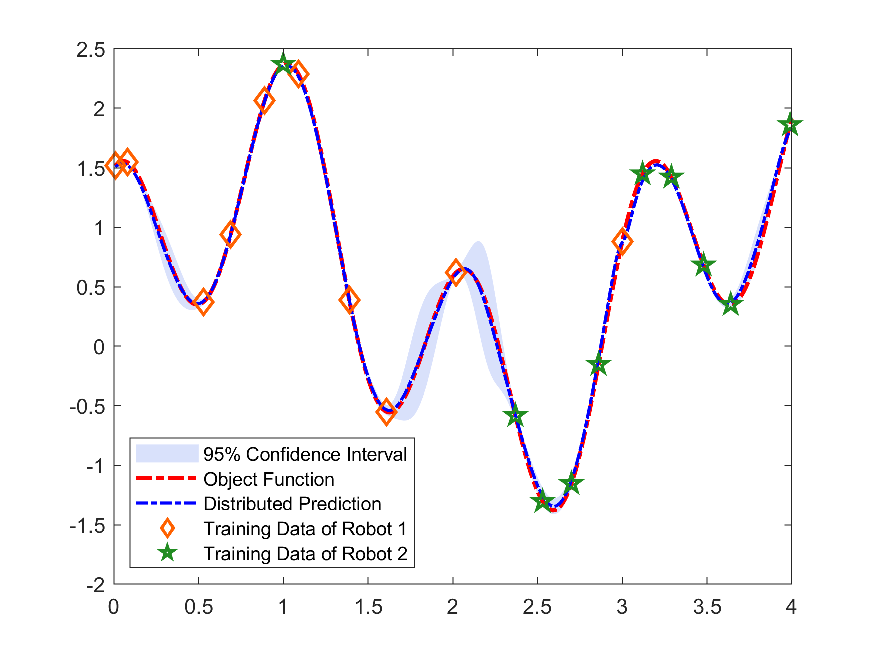}
	}
	\caption{A 1-D toy example with different compression methods. There are 2 robots sampling and predicting the objective function, respectively, where Robot~1 samples points ($N=300$) from $x_{1,0}=3$ to $x_{1,300}=0$ and Robot~2 samples from $x_{2,0}=1$ to $x_{2,300}=4$. The shaded area represents $95\%$ of the distributed GP confidence intervals. The diamond and star represent the sparse training data ($M=10$) after the online compression. The toy objective function is $f(x) = \sin(2x)+\cos(6x)+0.5$. (a) The centralized compression requires all real-time sample data. (b) Each robot compresses sample data locally, causing information repeated and globally inefficient. (c) Each robot compresses sample data based on distributed predictions with better performance. }
	\label{fig:1D_example}
\end{figure*}

\subsection{Distributed sparse metrics for recursive Gaussian process}

For a robot $i$, the recursive variable scalar $q_{i,t+1}$ and $r_{i,t+1}$ \eqref{eq:q,r} for can be interpreted as the variance-weighted prediction at point $\bxx_{i,t+1}$ using datasets $\bD_{i,t}$. Ref. \cite{kepler_approach_2020} use the sum of the changes in posterior mean at all current sample points $\{\bxx_{i,1},...,\bxx_{i,t}\}$ to denote the score for the observation $\{\bxx_{i,t+1},~y_{i,t+1}\}$, which has been derived as
\begin{equation} \label{local score}
	\epsilon_{i,t+1} = \left\lvert\frac{ \balpha_{i,t+1}^{[t+1]}}{\mathbf{diag}(\bQ_{i,t+1})^{[t+1]}} \right\rvert,
\end{equation}
where $[t+1]$ denotes the $(t+1)$-th item of the vectors and $\mathbf{diag}(\cdot)$ denotes the diagonal vectors of the matrix. However, for a multi-robot system, each robot $i$ has different datasets $\bD_{i,t}$ and recursive Gaussian process variables $\balpha_{i,t},~\bC_{i,t}$. \red{The score \eqref{local score} can only evaluate the points based on local information, which causes repeated information and limited global performances.} A 1-D toy example is shown in Fig.~\ref{fig:1D_example}. Two robots sample the 1-D toy function and compress observations. Compression based on only local information tends to retain repeated information in common areas because of the same local metrics.

In Section~III, we have obtained the distributed aggregation results based on consensus methods. And then we employ the the distributed aggregation results $\rho_{i,t}^{[D]}$ to improve the compression performances.

Bhattacharyya-Riemannian distance is employed to evaluate the sample points using the distributed aggregation results. \red{For two Gaussian distributions $\nu_1(\mathbf{x}) = \mN(\bmu_1,\bSigma_1)$ and $\nu_2(\bxx) = \mN(\bmu_2,\bSigma_2)$ over feature space $\mX$, Bhattacharyya-Riemannian distance is defined as }
\begin{equation} \label{BR metric}
	\begin{aligned}
		&d_\mathcal{BR}(\nu_1,\nu_2) =  d_\mathcal{B}(\nu_1,\nu_2) + d_\mathcal{R}(\bSigma_1,\bSigma_2), \\
		&d_\mathcal{B}(\nu_1,\nu_2) = \sqrt{(\bmu_1-\bmu_2)^T \bar{\bSigma}^{-1} (\bmu_1-\bmu_2)},\\
		&d_\mathcal{R}(\bSigma_1,\bSigma_2) = \sqrt{\sum_{j=1}^{n_\lambda} (\log \lambda_j)^2 },
	\end{aligned}
\end{equation}
where $\bar{\bSigma} = \frac{1}{2} (\bSigma_1 + \bSigma_2)$, and $\{\lambda_1,...,\lambda_{n_\lambda}\}$ is the  eigenvalues of $\bSigma_2^{-1}\bSigma_1$. This metric \eqref{BR metric} modifies the Hellinger distance by employing Riemannian item $d_\mathcal{R}$, which is efficient for the Gaussian variance representation \cite{abou-moustafa_note_2012}.

Denote $\bS_{i,k}:= \{\bxx_{i,k},~y_{i,k}\} \subset \bD_{i,t}$ as the evaluated observation. 
We design a distributed sparse metrics as
\begin{equation} \label{metric}
	\phi(\bS_{i,k}) := k_{\phi} \Omega \left(- d_\mathcal{BR}(\rho_{i,t}^{[D]},\widetilde{\rho}_{i,t,k}^{[L]}) \right)  + (1-k_{\phi}) \Omega \left( \epsilon_{i,k} \right),
\end{equation} 
\begin{equation} \label{test_posterior}
	\widetilde{\rho}_{i,t,k}^{[L]} := \mP(f(\bxx^*)|\bD_{i,t} \setminus \bS_{i,k},\bxx^*),
\end{equation}
where $k_{\phi} \in (0,1)$ is a constant weight for the local score and distributed score and $\Omega(\cdot)$ is the normalization function for all evaluated observations. This letter uses the simple min-max normalization. $\epsilon_{i,k}$ denotes the local score \eqref{local score} for robot $i$ about the evaluated observation $\bS_{i,k}$. $\widetilde{\rho}_{i,t,k}^{[L]}$ is the local recursive Gaussian process distribution removing the evaluated observation, which will be introduced latter.

In Section~III, we have proved the distributed mean will converge to the centralized PoE results. Therefore, the distributed sparse metrics \eqref{metric} can find observations of sufficient global utility. 

\begin{algorithm} [t]
	\caption{Cooperative Online Filed Mapping Algorithm for Robot $i$} \label{alg:1}
	\begin{algorithmic}[1]
		\For{$t=1,2,...$}
		\State /* \textit{Sensing and Local Update} */
		\State Sample independent training point $\{\bxx_{i,t},y_{i,t}\}$
		\State $\bD_{i,t} = \{\bD_{i,t-1};(\bxx_{i,t},y_{i,t})\},~\tilde{M}_i = \textbf{Size}(\bD_{i,t})$
		\State Update local Gaussian process predictions \eqref{RLGPR}
		\begin{center}
			$\rho_{i,t}^{[L]} =\mP(f(\bxx^*)|\bD_{i,t},\bxx^*) = \mN(\bmu_{i,t}^{[L]},~\bSigma_{i,t}^{[L]})$
		\end{center}
		\State /* \textit{Communication and Achieving Consensus} */
		\State Receive $\bxi_{j,t}$ from neighbors $j \in \mathbb{N}_{i,t}$
		\State Compute local consensus state $\bxi_{i,t+1}$ \eqref{eq:DGPR}
		\State Update distributed Gaussian process predictions \eqref{dist_update} 
		\begin{center}
			$\rho_{i,t}^{[D]} = \mN(\bmu_{i,t}^{[D]},~\bSigma_{i,t}^{[D]})$
		\end{center} 
		\State Send $\bxi_{i,t+1}$ to neighbors $j \in \mathbb{N}_{i,t+1}$
		\State /* \textit{Distributed Data Compress} (Algorithm~2) */
		\If{DataSize  $\tilde{M}_i > M_i$}
		\State $\bD_{i,t} = \textbf{Compress}\left(\bD_{i,t},\rho_{i,t}^{[D]}\right)$ 
		\EndIf
		\EndFor
	\end{algorithmic}
\end{algorithm}

\subsection{Distributed sparse Gaussian process update}
Then, we employ the distributed sparse metrics \eqref{metric} to select the information-efficient subsets with $M$ points and design the distributed sparse Gaussian process.

Reorder the elements of ${\balpha}_{i,t},~\bC_{i,t},~\bQ_{i,t}$ so that they are consistent with $\bS_{i,k}$ being the last element in the current datasets $\bD_{i,t}$, and define
\begin{equation*}
    \tilde{\balpha}_{i,k} = \begin{bmatrix}
			\balpha_{i,t}^{[\tilde{k}]} \\ \balpha_{i,t}^{[k]}
    \end{bmatrix} = \begin{bmatrix}
        \tilde{\balpha}_{i,k}^{[1:M]} \\ \alpha^*_{i,k}
    \end{bmatrix}, \\
\end{equation*}
\begin{equation} 
	\begin{aligned}
		\tilde{\bC}_{i,k} &= \begin{bmatrix}
			\bC_{i,t}^{[\tilde{k},\tilde{k}]} & \bC_{i,t}^{[\tilde{k},k]} \\ \bC_{i,t}^{[k,\tilde{k}]} & \bC_{i,t}^{[k,k]}
		\end{bmatrix} = \begin{bmatrix}
			\tilde{\bC}_{i,k}^{[1:M,1:M]} & \bc_{i,k} \\ \bc^T_{i,k} & c^*_{i,k}
		\end{bmatrix}, \\
		\tilde{\bQ}_{i,k} &= \begin{bmatrix}
			\bQ_{i,t}^{[\tilde{k},\tilde{k}]} & \bQ_{i,t}^{[\tilde{k},k]} \\ \bQ_{i,t}^{[k,\tilde{k}]} & \bQ_{i,t}^{[k,k]}
		\end{bmatrix} = \begin{bmatrix}
			\tilde{\bQ}_{i,k}^{[1:M,1:M]} & \bq_{i,k} \\ \bq^T_{i,k} & q^*_{i,k}
		\end{bmatrix},
	\end{aligned}
\end{equation}
where $\bC^{[i,j]}$ denotes the elements of the matrix $\bC$ at $i$-th row and $j$-th column and $\tilde{k} = [1,\dots,k-1,k+1,\dots,M+1]$ denotes an index vector from $1$ to $M+1$ \red{except} $k$. The recursive variables removing points $\bS_{i,k}$ are given as \cite{csato_sparse_2002}
\begin{equation} \label{reduce}
	\begin{aligned}
		\hat{\balpha}_{i,k} &= \tilde{\balpha}_{i,k}^{[1:M]} - \frac{\alpha^*_{i,k}}{q^*_{i,k}} \bq_{i,k}, \\
		\hat{\bC}_{i,k} &= \tilde{\bC}_{i,k}^{[1:M,1:M]} + \frac{c^*_{i,k}}{(q^*_{i,k})^2} \bq_{i,k} \bq^T_{i,k} - \frac{1}{q^*_{i,k}}\mathbf{\Gamma}_{i,k}, \\
		\hat{\bQ}_{i,k} &= \tilde{\bQ}_{i,k}^{[1:M,1:M]} - \frac{1}{q^*_{i,k}} \bq_{i,k} \bq^T_{i,k},
	\end{aligned}
\end{equation}
where $\mathbf{\Gamma}_{i,k} = \bq_{i,k} \bc^T_{i,k} + \bc_{i,k} \bq^T_{i,k}$, so as to compute $\widetilde{\rho}_{i,t,k}^{[L]}$ using the reduced $\hat{\balpha}_{i,k},~\hat{\bC}_{i,k},~\hat{\bQ}_{i,k}$ by \eqref{RLGPR}.


\subsection{Resource-efficient cooperative online mapping algorithm and complexity analysis}

The entire cooperative online field mapping method is shown in Algorithm~1 and the pipline is presented in Fig.~\ref{fig:flowfigure}. At the time $t$, each robot samples at the field and updates local Gaussian process predictions $\rho_{i,t}^{[L]}$ by \eqref{RLGPR}. Next, robots share local consensus state $\bxi_{i,t}$ \eqref{eq:DGPR} with neighbors and update distributed Gaussian process predictions $\rho_{i,t}^{[D]}$ by \eqref{dist_update}. Moreover, robots use the distributed sparse metrics \eqref{metric} to greedily select the information-efficient subsets (Algorithm~2). The recursive variables $\balpha_{i,t},~\bC_{i,t},~\bQ_{i,t}$ keep lower $M$ dimensions by \eqref{reduce} for fast Gaussian process predictions.

\begin{algorithm}[t]
	\caption{Distributed Data Compress for Robot $i$} \label{alg:compress}
	\begin{algorithmic}[1]
		\For{$k = 1,...,M_i+1$}
		\State Compute test Gaussian posterior $\widetilde{\rho}_{i,t,k}^{[L]}$ \eqref{test_posterior} without point $\bS_{i,k} = \{\bxx_{i,k},y_{i,k}\}$
		\State Obtain distributed sparse metrics $\phi(\bS_{i,k})$ \eqref{metric}
		\EndFor
		\State Find minimal information sample index
		\begin{center}
			$k^* = \operatorname{argmin}_{k\in\bD_{i,t}}\phi^*_k$
		\end{center} 
		\State Remove points $\{\bxx_{i,k^*},y_{i,k^*}\}$ from datasets 
		\begin{center}
			$\bD_{i,t}= \bD_{i,t}\setminus \{\bxx_{i,k^*},y_{i,k^*}\}$
		\end{center} 
		\State Update recursive variables with $M$ dimensions \eqref{reduce}
		\begin{center}
			$\balpha_{i,t} = \hat{\balpha}_{i,k^*},~\bC_{i,t}= \hat{\bC}_{i,k^*},~\bQ_{i,t}= \hat{\bQ}_{i,k^*}$
		\end{center} 
		\State \Return $\bD_{i,t},~\balpha_{i,t},~\bC_{i,t},~\bQ_{i,t}$
	\end{algorithmic}
\end{algorithm}
A comparison of time, memory, and communication complexity is presented in Table~\ref{complex}. Recall that the robot number is $p$, and each robot samples $N$ training points and selects $M$ sparse points. \red{The numbers of test points $\bxx^*$ and communication edges are $\mathcal{T}$ and $|E|$}. 

\red{\textbf{Computation and memory:} for each robot, the recursive sparse update of local mapping yields $\mathcal{O}\left(N(M^2+M^3)\right)$ time and $\mathcal{O}(M^2)$ memory complexity, which is significantly less than distributed recursive GP with $\mathcal{O}(N^3)$ time and $\mathcal{O}(N^2)$ memory complexity. Compared with local sparse metrics \eqref{local score}, the proposed distributed sparse metrics \eqref{metric} have better global accuracy with additional $\mathcal{O}\left(NM^3\right)$ computation. Since $M \ll N$, the additional computation is acceptable for online mapping.}

\red{\textbf{Communication:} compared with communicating original observations $\mathcal{O}(p^2N)$ or compressed observations $\mathcal{O}(p^2M)$ \cite{kepler_approach_2020}, \cite{zobeidi_dense_2022}, the proposed methods only communicate consensus state $\bxi$ and have $\mathcal{O}(|E|\mathcal{T})$ communication complexity, where $p-1\leq |E| \leq p(p-1)/2$. Therefore, the proposed methods have lower communication costs for large-scale multi-robot systems where $p$ is large. Moreover, the proposed methods can avoid repeated computations based on the same observations for GP mapping in different robots.}

\begin{figure}[!t]
	\centering
	\includegraphics[scale=0.3,trim=50 30 0 0]{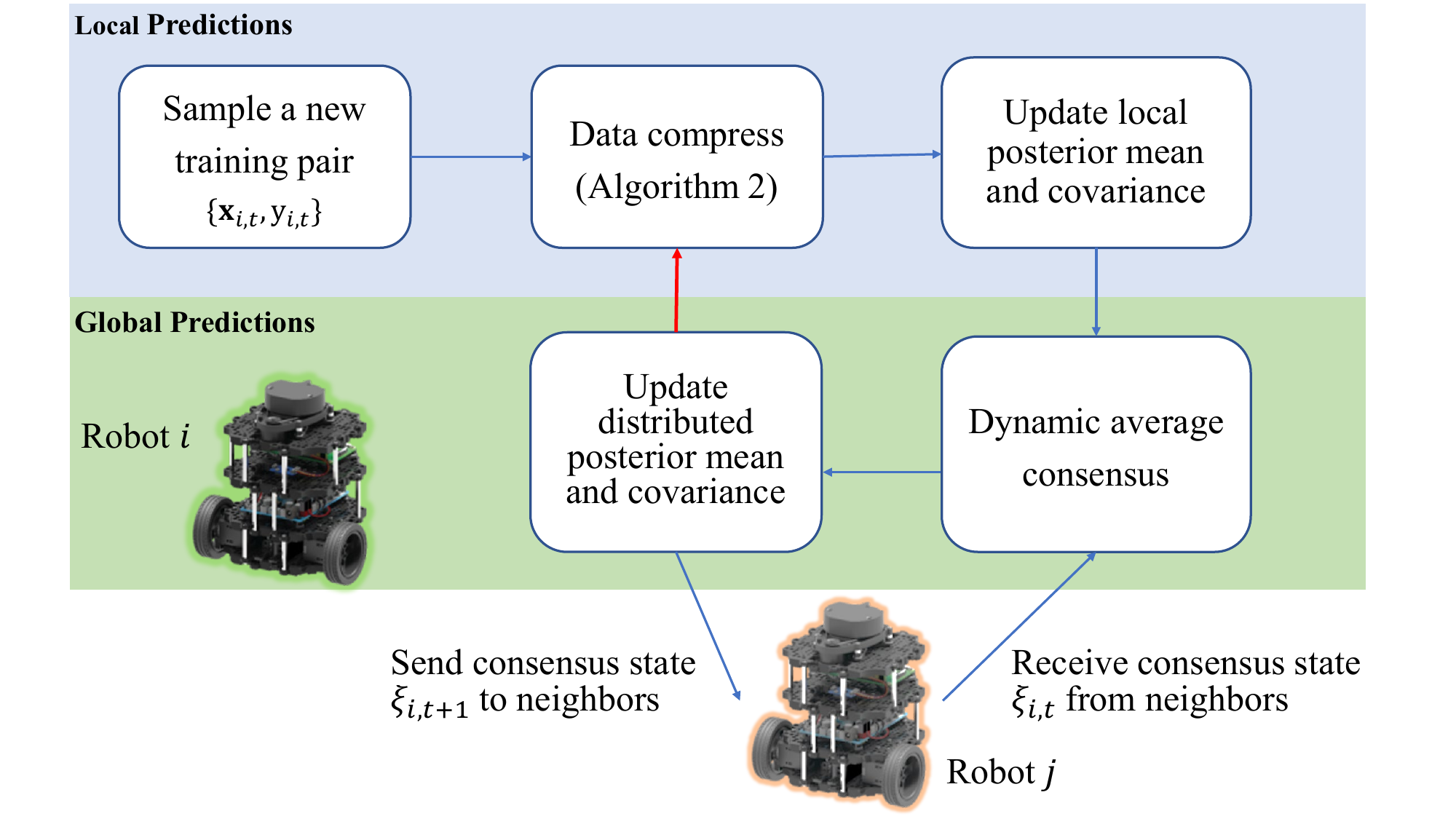}
	\caption{Pipeline of the proposed Algorithm~1. }
	\label{fig:flowfigure}
\end{figure}

\begin{table*} [t] 
	\centering 
	\caption{Algorithm complexity for $N$ streaming data and $p$ robots.}
	\label{complex}
	\renewcommand\arraystretch{1.5}
	\begin{tabular}{ccccccc}
		\hline
		\multirow{2}{*}{Complexity} & Full GPR & RGPR  & DGPR & DRGPR  & Local SGPR & \multirow{2}{*}{This letter} \\
		 & \cite{GPForML} & \cite{csato_sparse_2002} & \cite{yuan_communication-aware_2020}, \cite{lederer_cooperative_2023} & \cite{jang_multi-robot_2020} & \cite{kepler_approach_2020}, \cite{zobeidi_dense_2022} &  \\
		\hline
		\red{Computation} & $\mathcal{O}(p^4N^4)$ & $\mathcal{O}(p^3N^3)$ & $\mathcal{O}(pN^4)$ & $\mathcal{O}(pN^3)$  &$\mathcal{O}(p^3NM^2)$ & $\mathcal{O}\left(pN(M^2+M^3)\right)$ \\
		\hline
		\red{Memory} & $\mathcal{O}(p^2N^2)$ & $\mathcal{O}(p^2N^2)$ & $\mathcal{O}(pN^2)$ & $\mathcal{O}(pN^2)$ & $\mathcal{O}(p^3M^2)$ & $\mathcal{O}(pM^2)$ \\
		\hline
            \red{Communication} & \red{$\mathcal{O}(p^2N)$} & \red{$\mathcal{O}(p^2N)$} & \red{$\mathcal{O}(|E|\mathcal{T})$ } & \red{$\mathcal{O}(|E|\mathcal{T})$ } & \red{$\mathcal{O}(p^2M)$}  & \red{$\mathcal{O}(|E|\mathcal{T})$ }
	\end{tabular}
	\begin{tablenotes}
		\footnotesize
		\item[*] *Compared with local compression \eqref{local score}, the proposed method has better global accuracy with additional $\mathcal{O}\left(NM^3\right)$ computation for compression, which is acceptable since $M \ll N$.
	\end{tablenotes}
	
\end{table*}

\section{Experimental results}

In this section, the proposed algorithm is validated by experiments with multiple TurtleBot3-Burger robots. Each robot shown in Fig.~\ref{fig1} is equipped with the BH1750FVI light sensor, sampling and mapping the light field online in a $7.5~\mathrm{m} \times 5~\mathrm{m}$ workspace. \red{The sample rate of the light sensor is set as $30$~Hz.} The locations of robots are obtained by April-Tags. The entire system is implemented in the Robot Operating System (ROS).

In this experiment, 5 robots are performing an exploration to map the light field in the workspace. The communication network is dynamic and determined by the distance between robots. \red{The communication range and frequency between robots are set as $3$~m and $5$~Hz.} \blue{The local and distributed model update frequencies are set as $30$ and $5$~Hz, respectively}. The corresponding adjacency matrix of networks is $a_{ij,t} = 0.1~(j \in \mathbb{N}_{i,t})$, $a_{ii,t} = 1 - \sum_{j \in \mathbb{N}_{i,t}} a_{ij,t}$. This letter does not consider path planning and the sample paths of robots are set as a fixed prior. Gaussian process hyper-parameters are set $\sigma_f^2 = 1$, $\Sigma_\eta = [1/26~0;0~1/40]$. \red{The hyperparameters are typically selected by maximizing the marginal log-likelihood offline using $100$ samples as a priori knowledge}. Observation noises and correction biases are set $\sigma_e^2 =\sigma_n^2 = 0.1$. The distributed weight is set $k_\phi = 0.2$. All 5 robots sample a total of $N = 1056 \times 5$ training points.

\subsection{Convergence validation of distributed aggregation}
Select a $50 \times 50$ mesh grid in the workspace as the test points to construct the light field mapping. The size of sparse training subsets is set as $M = 13 \times 5$. The averaged distributed predictions of all 5 robots for all test points in the experiment and the centralized PoE results are shown in Fig.~\ref{fig:consensus error}. The communication network is dynamic and Robots $4$ and $5$ are out of communication ranges of Robots $1,~2$ and $3$ at first time. When sample number $n > 166$, Robots $4$ and $5$ begin sharing information in the communication network and the distributed aggregation results of all robots converge to the centralized PoE results.

\begin{figure} [!t]
	\centering
	\includegraphics[width=7.5cm,height=4.5cm,trim=0 0 0 20, clip]{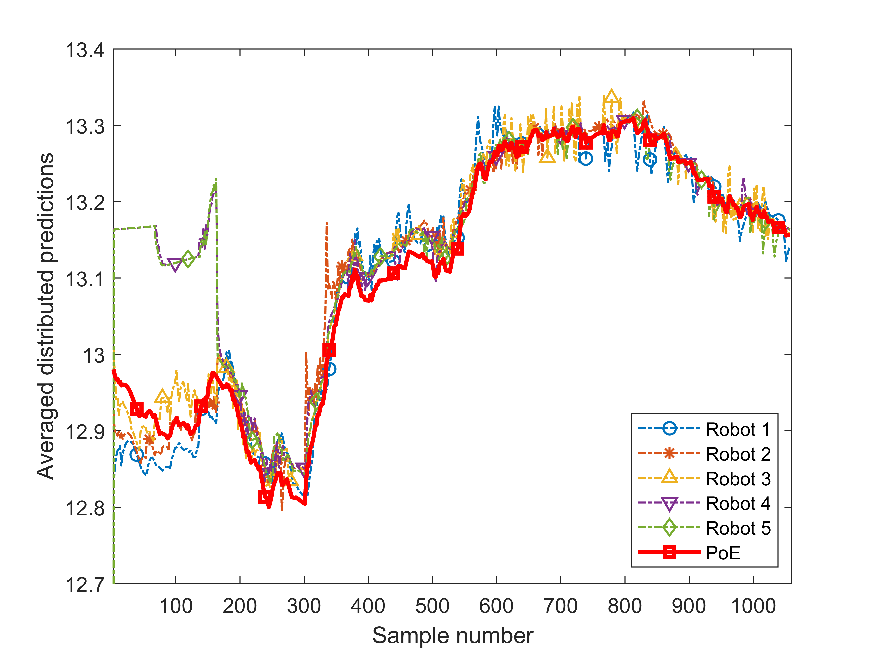}
	\caption{The averaged distributed predictions and centralized PoE results for all test points.}
	\label{fig:consensus error}
\end{figure}

Fig.~\ref{fig:field} represents the progress of this experiment and the corresponding distributed GP results of Robot $3$ during time $t=35$ ($n=1056$) seconds. Fig.~\ref{fig:field}(a) shows the experiment snapshots and the robot number. Fig.~\ref{fig:field}(b) shows the results of the distributed online GP variance predictions for Robot $3$ with the exploration trajectories of all robots. The black points denote the sparse subsets for GP predictions. Fig.~\ref{fig:field}(c) shows that the GP mean predictions for Robot $3$ gradually converges to the real light field in Fig.~\ref{fig:field}(a). Compared with the local GP results of Robot $3$ shown in Fig.~\ref{fig:Local results}, the distributed aggregation methods efficiently construct light field maps for unvisited areas.

\begin{figure*} [t]
	\flushright
	\subfloat{
		\begin{overpic}[width=4cm,height=3cm,trim=0 0 0 0, clip]{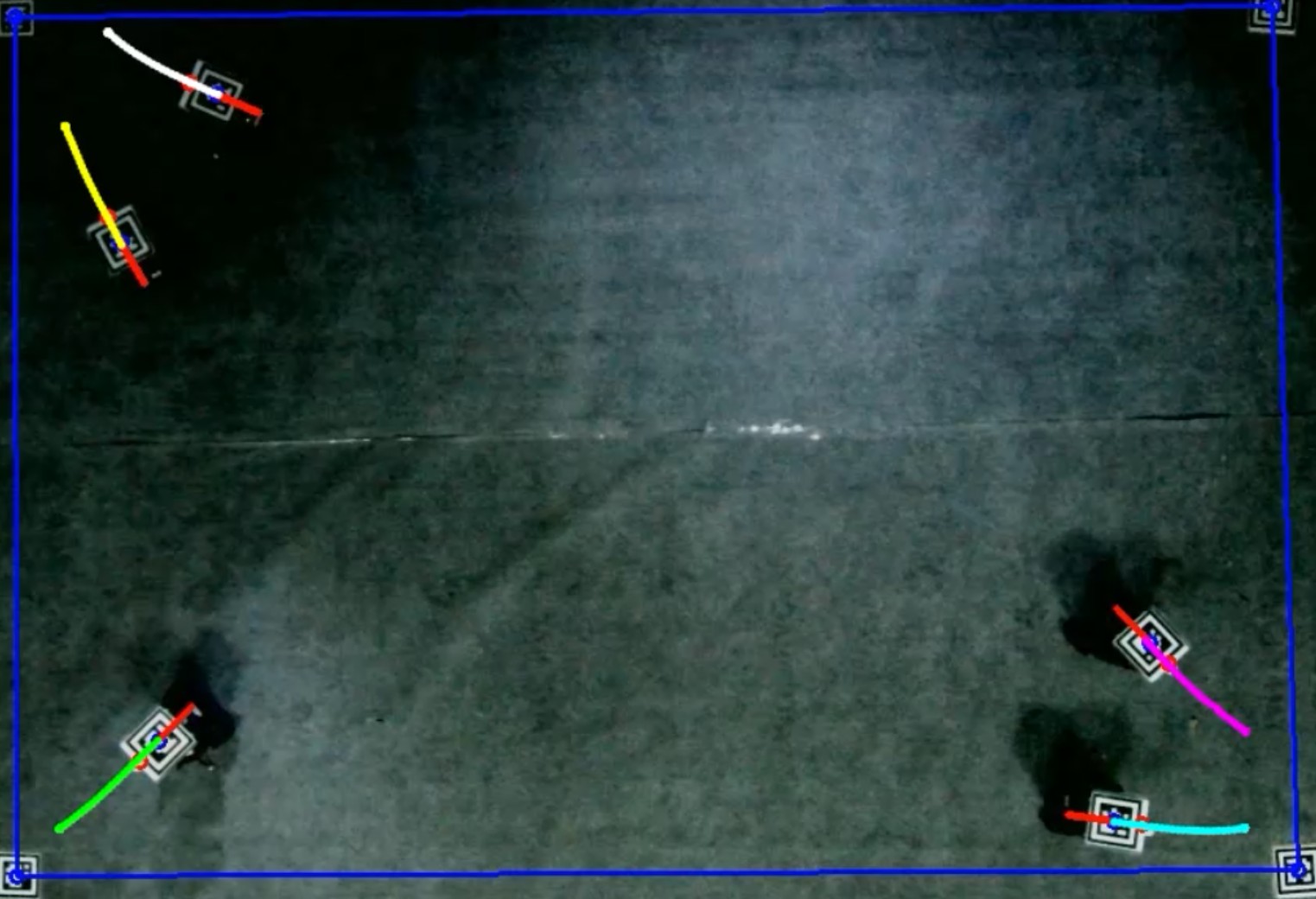}
			\put(-20,40){\color[rgb]{0,0,0}{(a)}} 
			\put(18,60){\color[rgb]{1,1,1}{R1}} 
			\put(8,44){\color[rgb]{1,1,1}{R2}} 
			\put(10,18){\color[rgb]{1,1,1}{R3}} 
			\put(72,10){\color[rgb]{1,1,1}{R4}} 
			\put(72,28){\color[rgb]{1,1,1}{R5}} 
		\end{overpic}
		\includegraphics[width=4.cm,height=3cm,trim=0 0 0 0, clip]{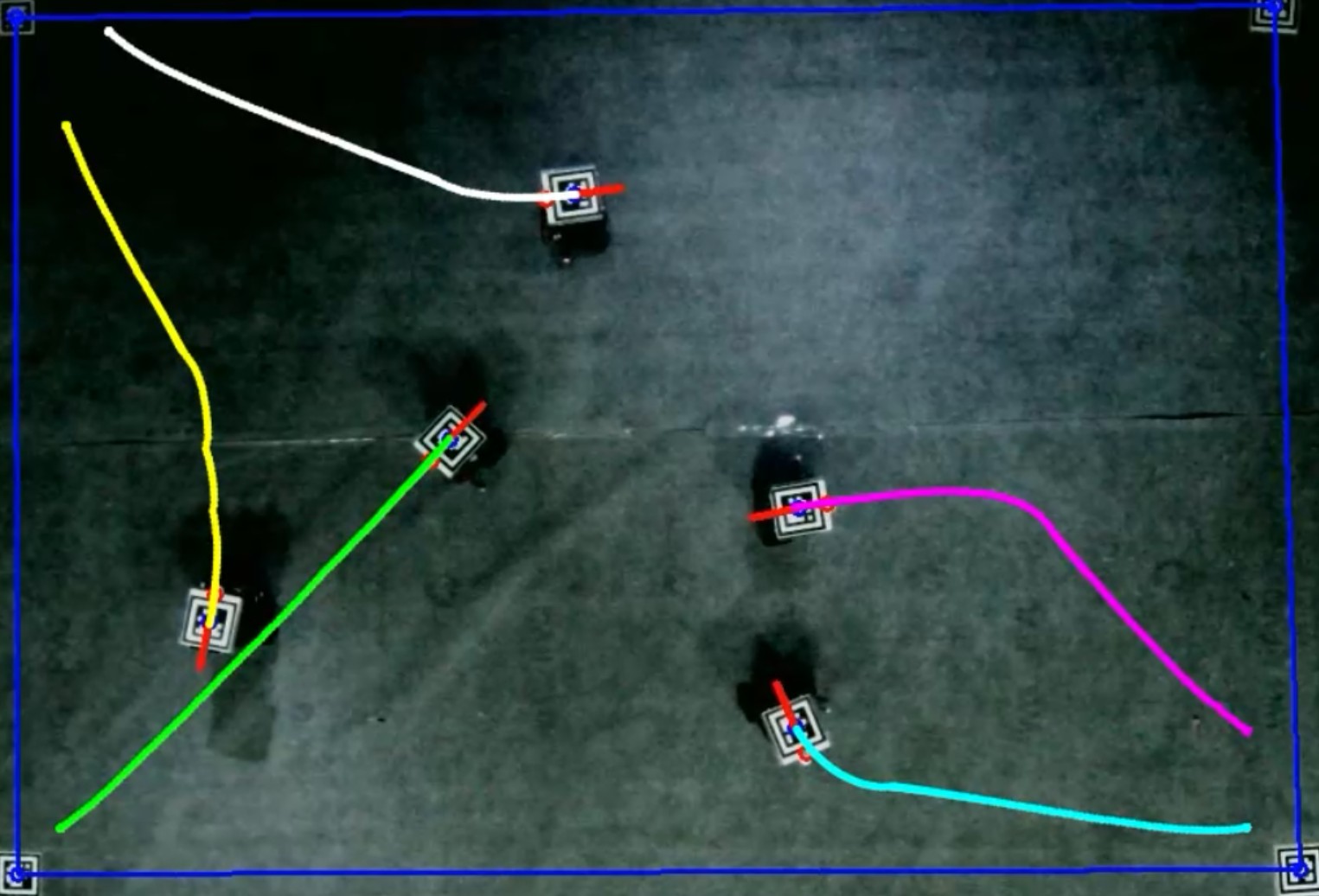}
		\includegraphics[width=4.cm,height=3cm,trim=0 0 0 0, clip]{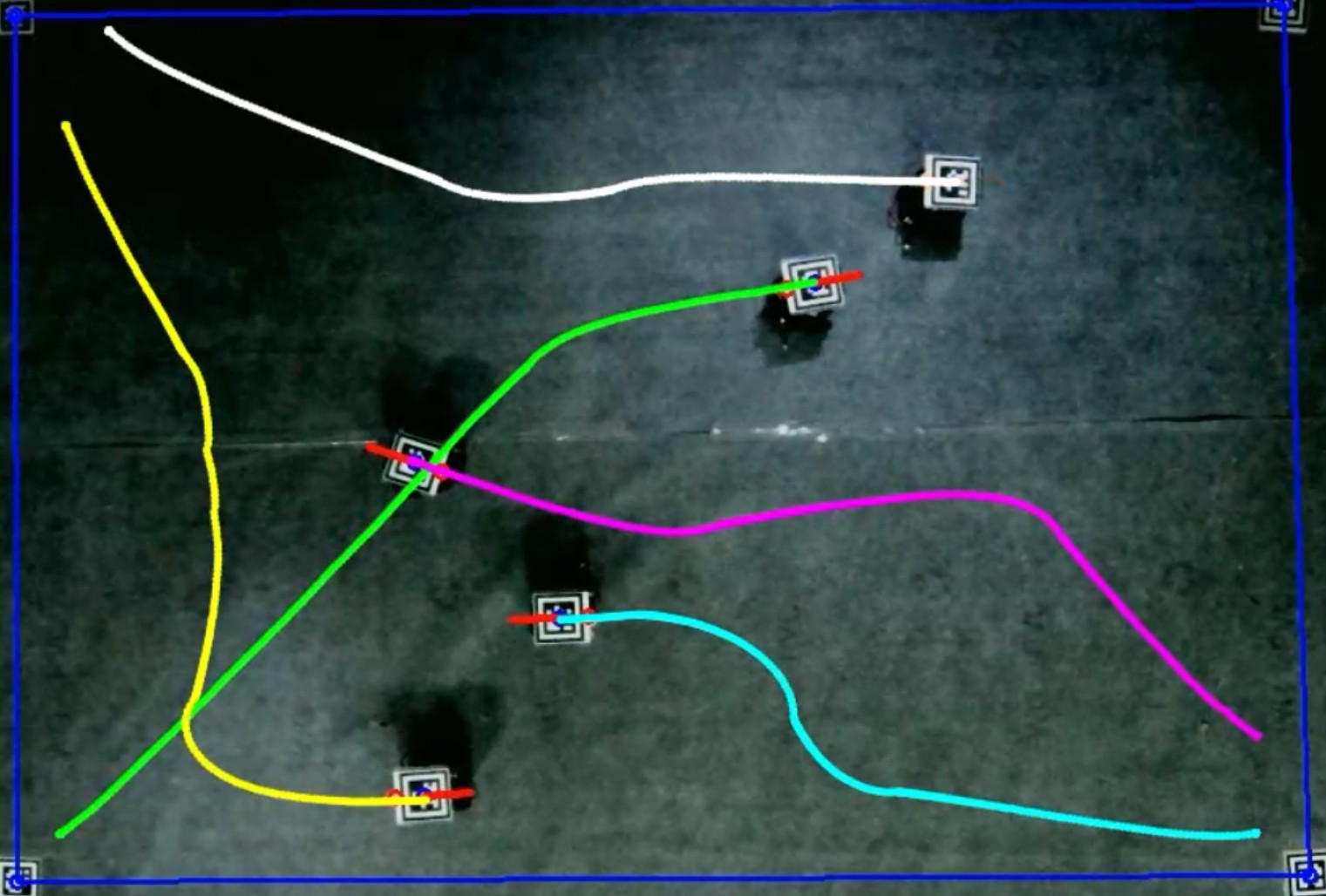}
		\includegraphics[width=4.cm,height=3cm,trim=0 0 0 0, clip]{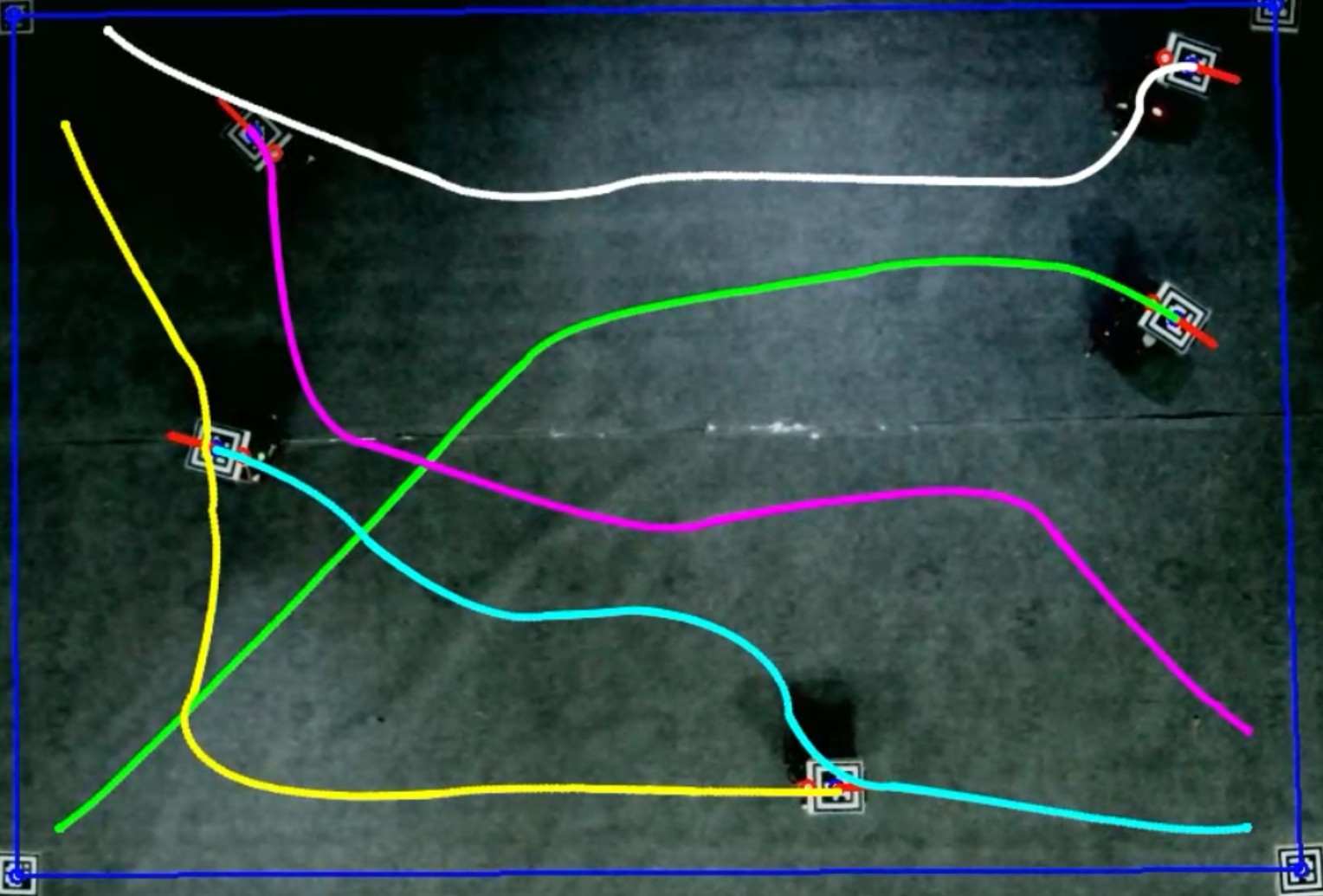}
		\includegraphics[width=0.36cm,height=3cm,trim=0 0 0.1cm 0, clip]{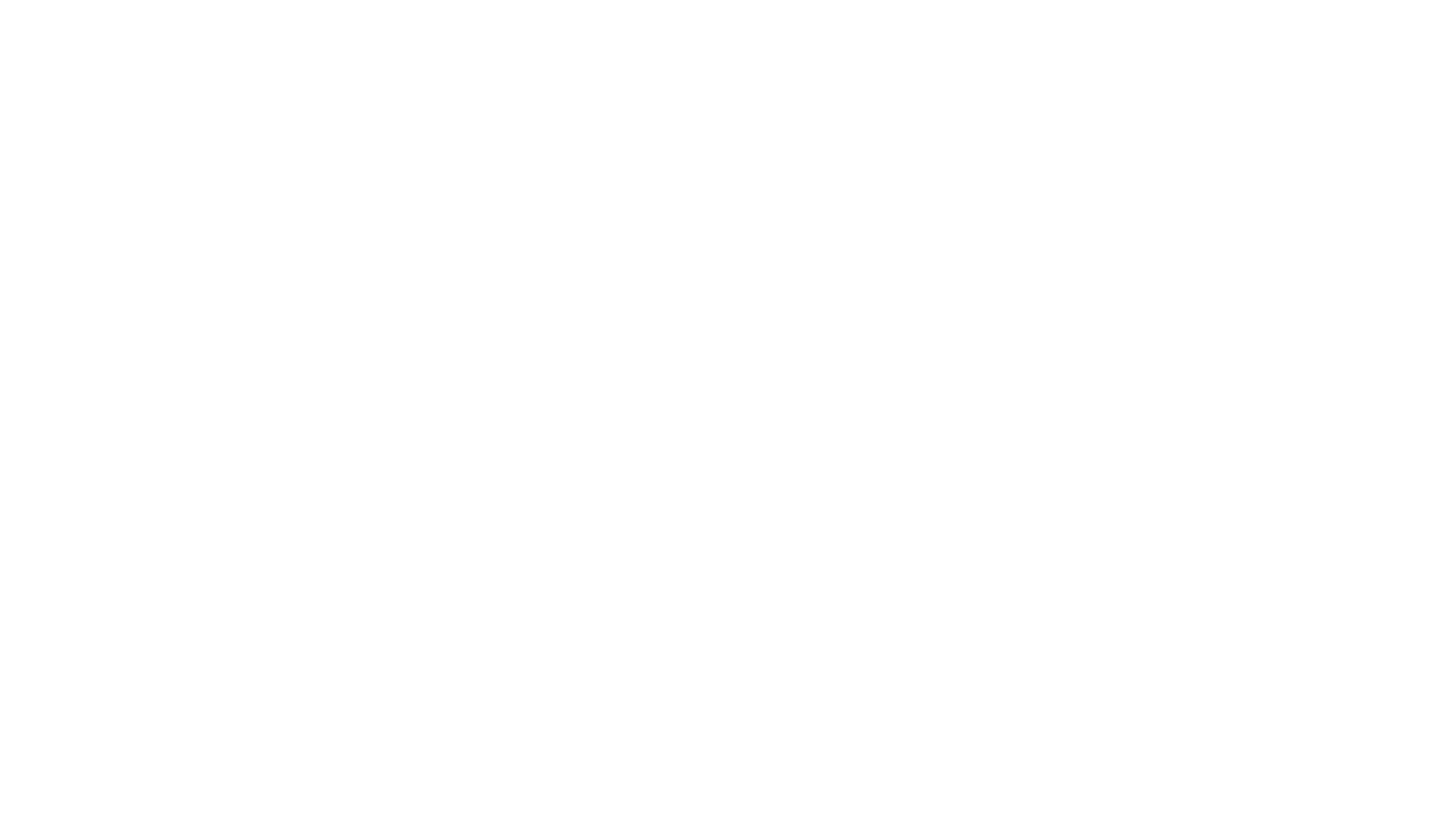}
	}  
	\\
	\vspace{-0.2cm}
	\subfloat{
		\begin{overpic}[width=4cm,height=3.5cm,trim=1.9cm 0 1.5cm 0.8cm, clip]{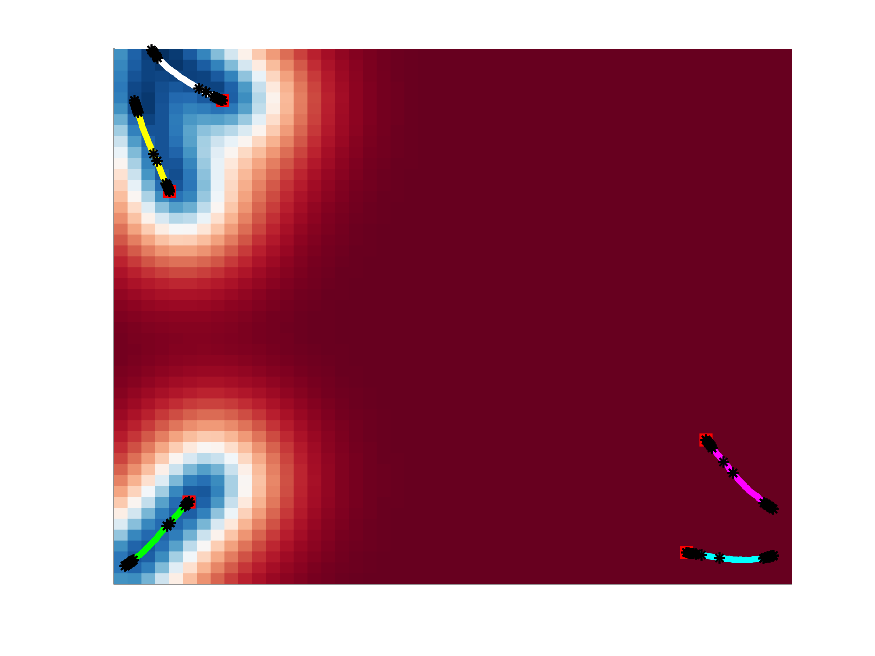}
			\put(-20,40){\color[rgb]{0,0,0}{(b)}} 
		\end{overpic}
		\includegraphics[width=4cm,height=3.5cm,trim=1.9cm 0 1.6cm 0.8cm, clip]{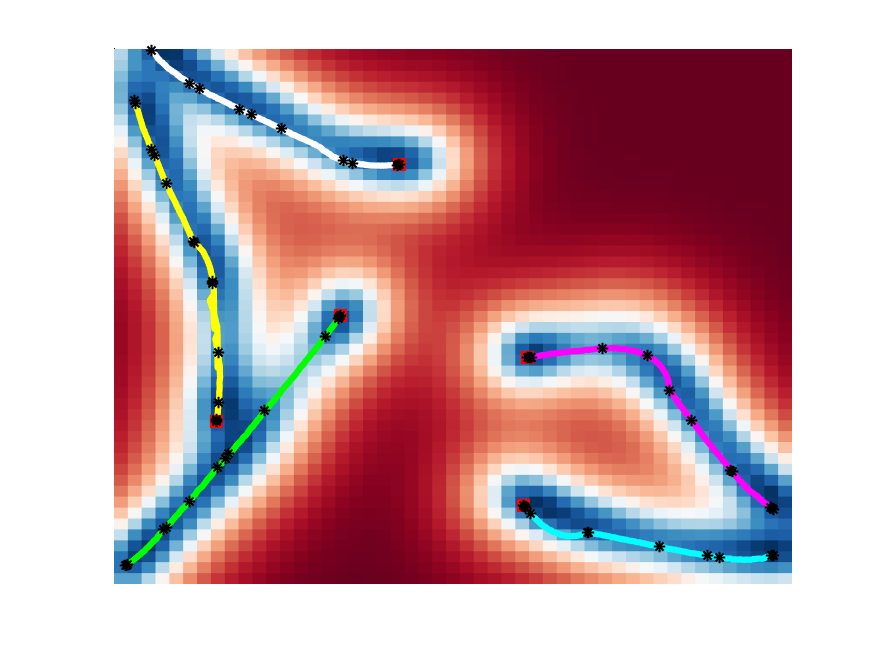}
		\includegraphics[width=4cm,height=3.5cm,trim=1.9cm 0 1.6cm 0.8cm, clip]{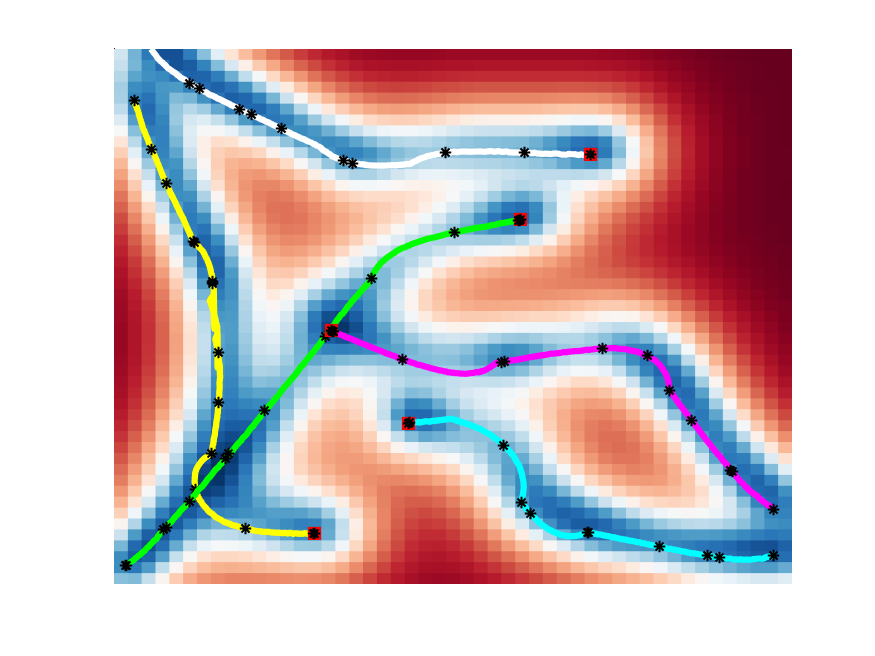}
		\includegraphics[width=3.88cm,height=3.5cm,trim=1.9cm 0 1.6cm 0.8cm, clip]{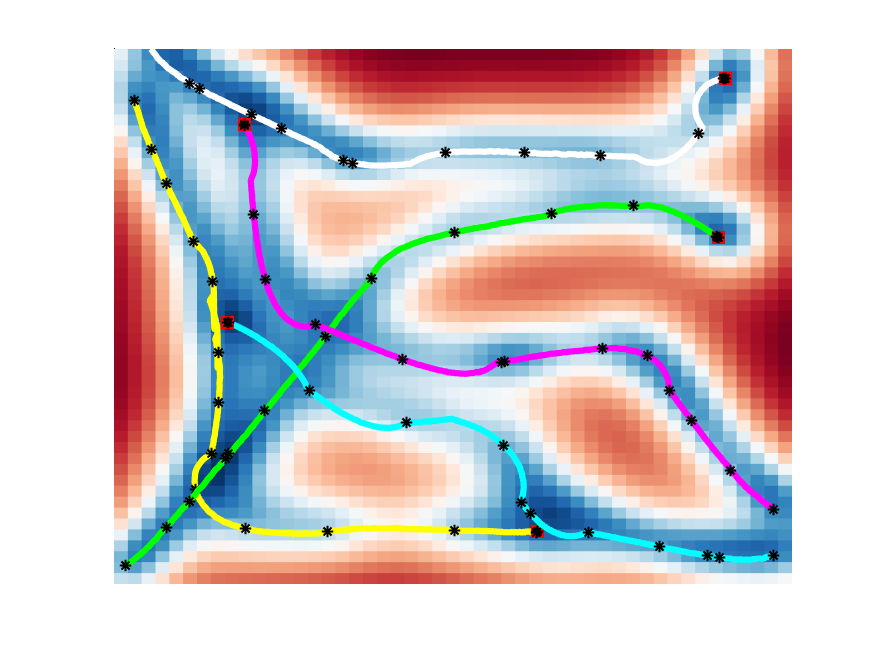}
		\includegraphics[width=0.5cm,height=3.5cm,trim=0.1cm 0 0cm 0.35cm, clip]{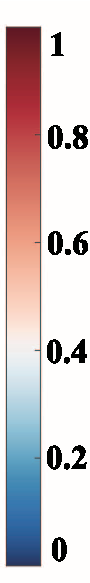}
	}  
	\\
	\vspace{-0.65cm}
	\subfloat{
		\begin{overpic}[width=4cm,height=3.5cm,trim=1.9cm 0 1.5cm 0.7cm, clip]{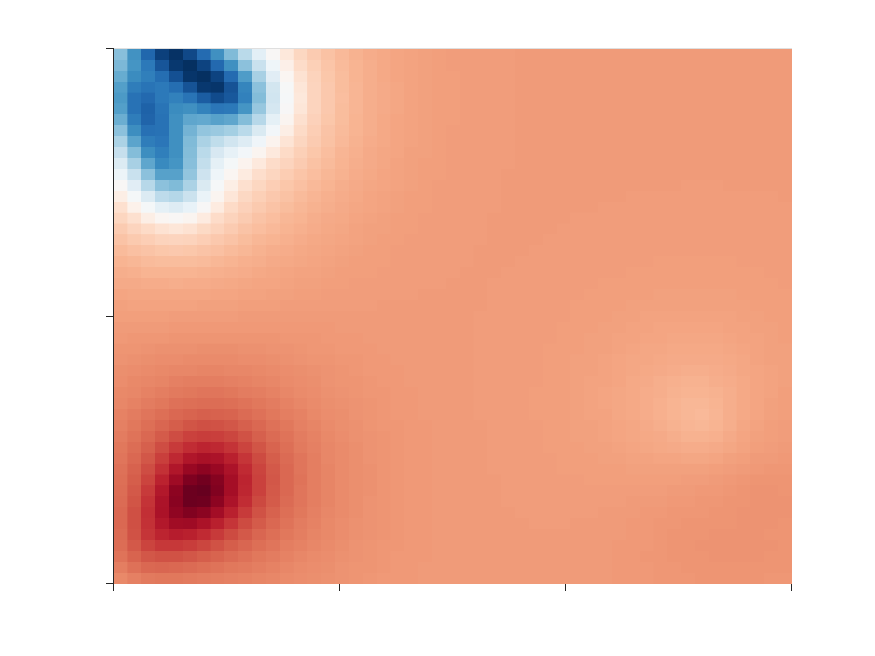}
			\put(-20,40){\color[rgb]{0,0,0}{(c)}} 
			\put(-18,55){y(m)}
			\put(45,-5){x(m)}
			
			\put(-6,82){\color[rgb]{0,0,0}{5}} 
			\put(-6,2){\color[rgb]{0,0,0}{0}} 
			\put(90,2){\color[rgb]{0,0,0}{7.5}} 
			\put(65,2){\color[rgb]{0,0,0}{5}} 
			\put(25,2){\color[rgb]{0,0,0}{2.5}} 
			\put(60,13){\color[rgb]{0,0,0}{$t=5s$}} 
		\end{overpic}
		\begin{overpic}[width=4cm,height=3.5cm,trim=1.9cm 0 1.6cm 0.7cm, clip]{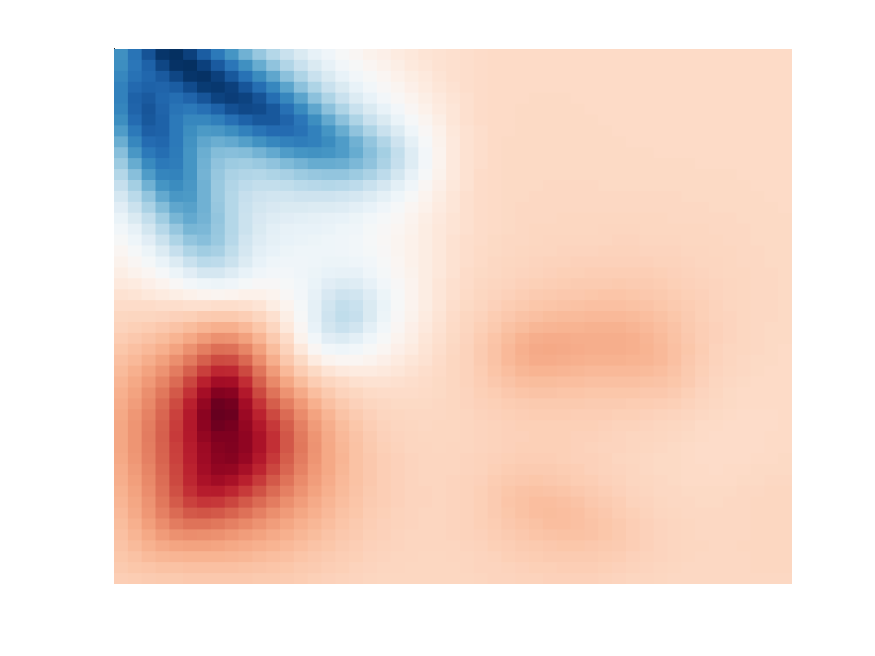}
			\put(60,13){\color[rgb]{0,0,0}{$t=15s$}} 
		\end{overpic}
		\begin{overpic}[width=4cm,height=3.5cm,trim=1.9cm 0 1.6cm 0.7cm, clip]{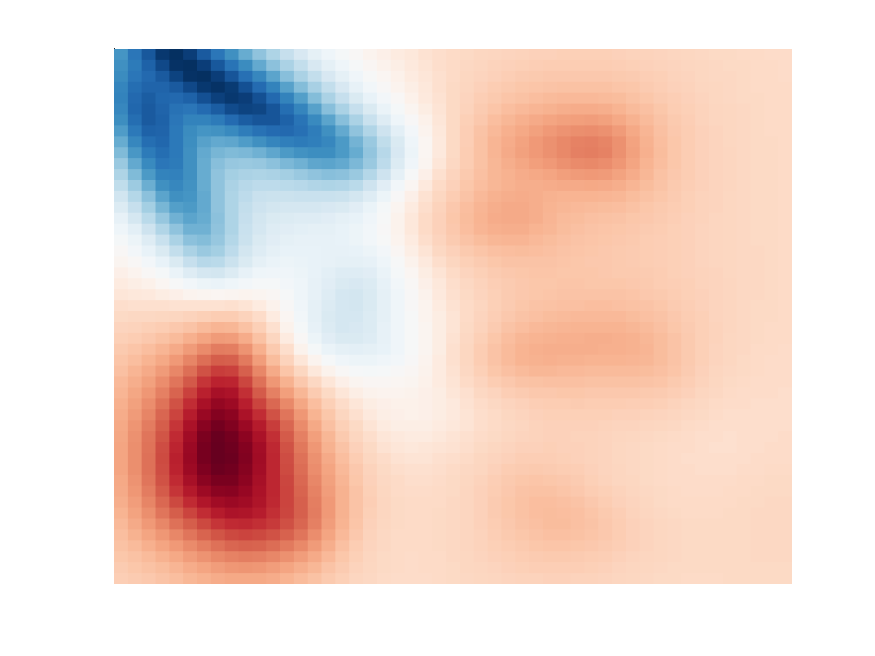}
			\put(60,13){\color[rgb]{0,0,0}{$t=25s$}} 
		\end{overpic}
		\begin{overpic}[width=4.5cm,height=3.5cm,trim=1.9cm 0 1.2cm 0.7cm, clip]{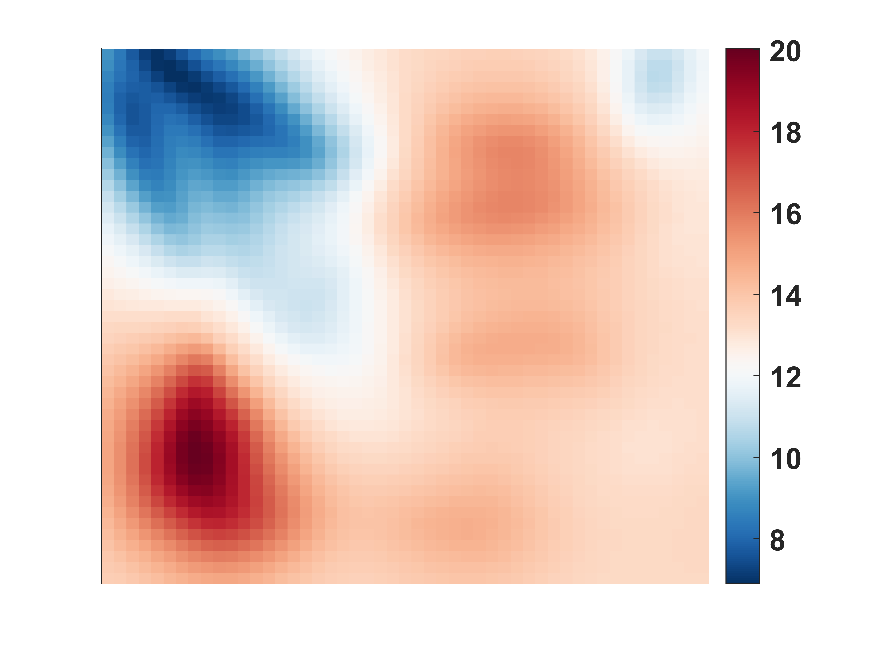}
			\put(53,12){\color[rgb]{0,0,0}{$t=35s$}} 
                \put(87,2){\color[rgb]{0,0,0}{(lx)}} 
		\end{overpic}
	}	
	\caption{A team of 5 robots cooperatively construct the online map of the light fields. (a) Experiment snapshots and the robot number. (b) Distributed online GP variance predictions for Robot $3$ and the exploration trajectories of all robots. The black points denote the sparse subsets for GP predictions. (c) Distributed online GP mean predictions for Robot $3$, which gradually converges to the real light field in (a).}
	\label{fig:field}
\end{figure*}

\begin{figure} [h]
	\centering
	\subfloat[Local GP variance $\bSigma_{3}^{[L]}$]{\includegraphics[width=4.2cm,height=3.2cm,trim=1.9cm 0 1cm 0.7cm, clip]{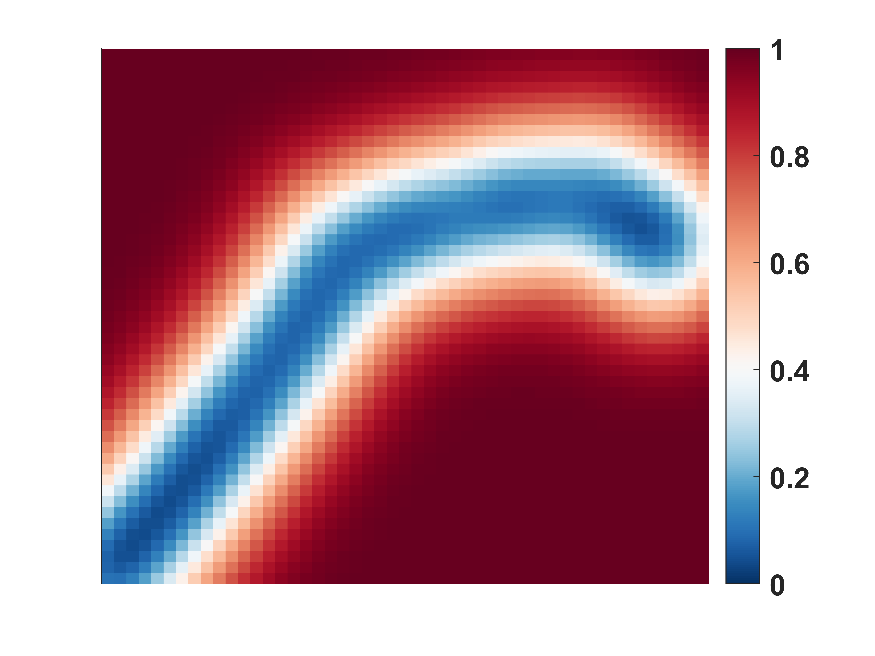}
	}
	\subfloat[Local GP mean $\bmu_{3}^{[L]}$]{\includegraphics[width=4.2cm,height=3.2cm,trim=1.8cm 0 1.2cm 0.7cm, clip]{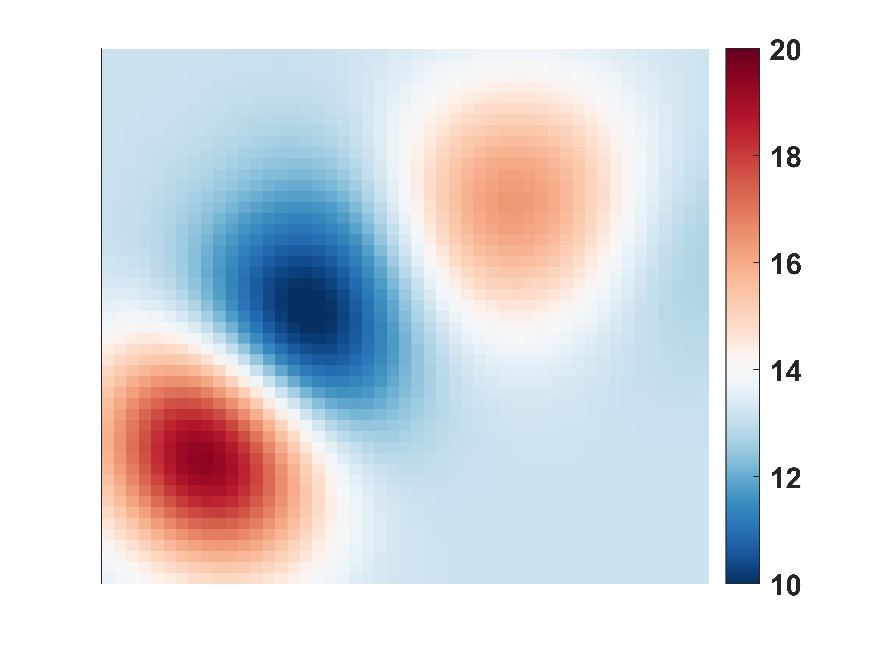}
	}
	\caption{Local Gaussian process mapping results of Robot $3$.  Compared with Fig.~\ref{fig:field}, the distributed aggregation methods efficiently construct light field maps for unvisited areas.}
	\label{fig:Local results}
\end{figure}
	
\subsection{Sparse prediction accuracy validation}
To analyze the map accuracy with different GP prediction methods, $10 \times 10$ equal interval points in the workspace are sampled by hand as the test points. Fig.~\ref{fig:performance}(a) shows the mean RMSE of the distributed predictions on the test points for all robots. Compared with the local data compression \cite{kepler_approach_2020}, the proposed distributed compression improves the prediction accuracy. Fig.~\ref{fig:performance}(b) shows the mean online prediction computational time for all robots on $20$ repeated experiments. The proposed sparse GP has an efficient constant prediction time with the increasing sample $N$. \red{In addition, the data compression is run all the time here to illustrate the constant computation time. In practice, it is not necessary when the sample number $N$ is small.}

\begin{figure}
	\centering
	\subfloat[RMSE]{
        \includegraphics[width=0.7cm,height=4cm,trim=0 0 0 0.5cm, clip]{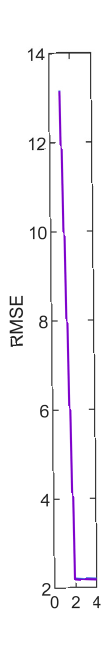}
        \includegraphics[width=5.2cm,height=4cm,trim=1.3cm 0 0 0.5cm, clip]{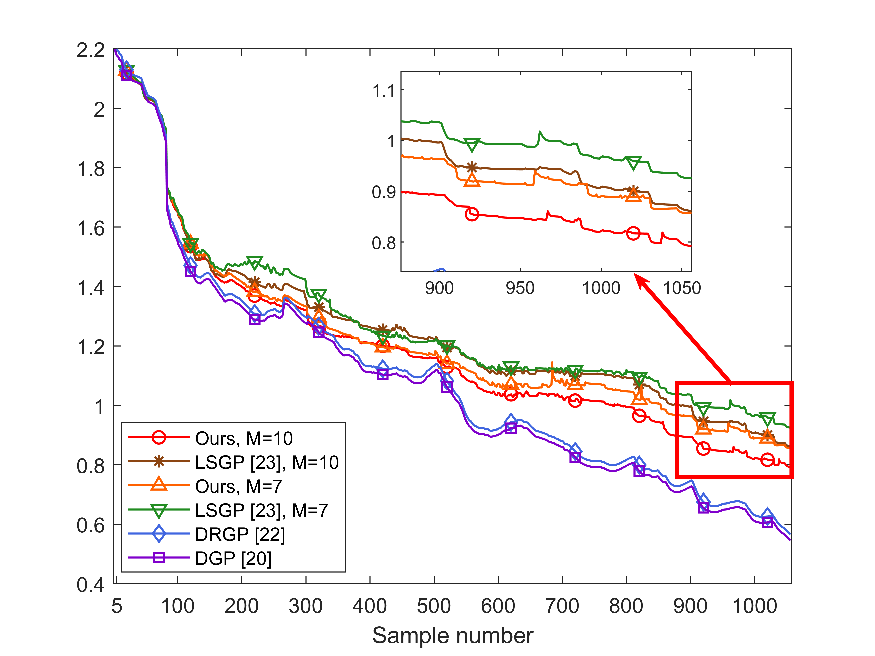}
	}
	\\
	\subfloat[Prediction Time]{
		\includegraphics[width=6cm,height=4cm,trim=0 0 0 0.6cm, clip]{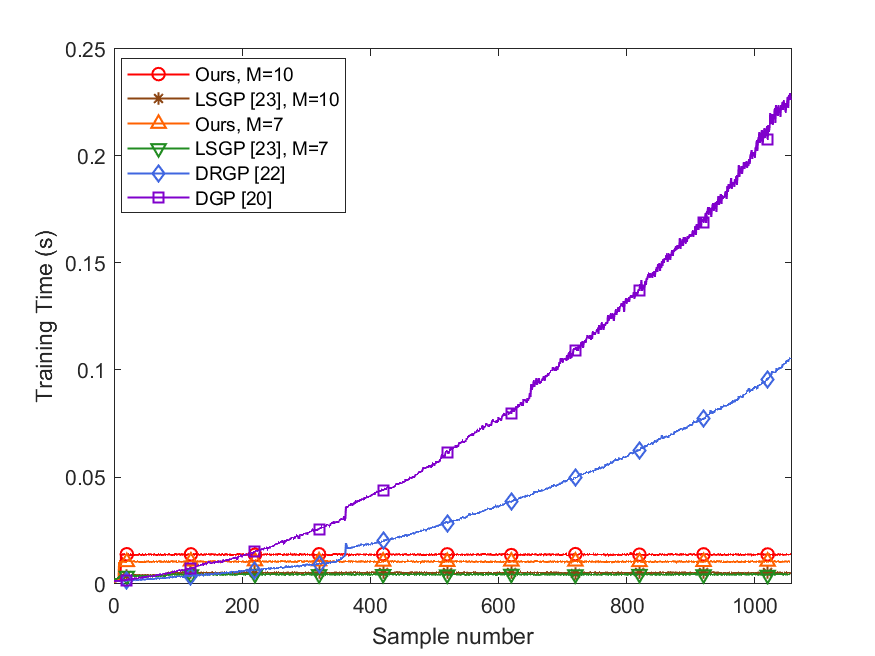}
	}
	\caption{Performance of the distributed GP \cite{lederer_cooperative_2023}, the distributed recursive GP \cite{jang_multi-robot_2020}, the local sparse GP \cite{kepler_approach_2020} and the proposed Algorithm~1 in this experiment. The proposed algorithm has an efficient constant prediction time with better global accuracy.}
	\label{fig:performance}
\end{figure}

\section{CONCLUSIONS}

This letter proposes a resource-efficient cooperative online field mapping method via distributed sparse Gaussian process regression. Novel distributed online sparse metrics are developed such that robots can cooperatively evaluate observations with better global accuracy. The bounded errors of distributed aggregation results are guaranteed theoretically, and performances of proposed algorithms are validated by online light field mapping experiments. Our future work will focus on the distributed sparse GP hyper-parameters training and variance-based sample path planning.

\section*{APPENDIX: proof of Theorem~1}
	Without loss of generality, we consider the test data $\bxx^*$ is a point here such that $\bmu(\bxx^*) \in \mathbb{R},~ \bSigma(\bxx^*) \in \mathbb{R}$. For simplicity, we omit the test points $(\bxx^*)$ in the following Gaussian process descriptions.
	We first derive the bounds of the reference input $\Vert \Delta\br_{i,t} \Vert$ in \eqref{eq:DGPR}.
	\begin{equation} \label{dr_1}
        \begin{aligned}
            \Delta\br_{i,t+1} &= \begin{bmatrix}
			\Delta\br_{i,t+1}^{[1]} \\   \Delta\br_{i,t+1}^{[2]} 
		\end{bmatrix} \\ &= \begin{bmatrix}
			\bmu_{i,t+1}^{[L]} \cdot  (\hat{\bSigma}_{i,t+1}^{[L]})^{-1} - \bmu_{i,t}^{[L]} \cdot (\hat{\bSigma}_{i,t}^{[L]})^{-1} \\   (\hat{\bSigma}_{i,t+1}^{[L]})^{-1} -  (\hat{\bSigma}_{i,t}^{[L]})^{-1} 
		\end{bmatrix}.
        \end{aligned}
	\end{equation}
	According to Ref. \cite{GPForML} and the SE kernel \eqref{eq:kernel}, the local Gaussian posterior variances have the bounds $0<\bSigma_{i,t}\leq \sigma_f^2$. Therefore, we derive
	\begin{equation}
			\sigma_n^2 \leq \hat{\bSigma}_{i,t+1}^{[L]} < \sigma_n^2+\sigma_f^2,
	\end{equation}
	Denote $K_*(\bxx_{t+1}) := K(\bxx^*,\bxx_{t+1})$ and $K_t(\bxx_{t+1}) := K(\bx_t,\bxx_{t+1})$. The recursive updates of $\balpha_{t+1}$ and $\bC_{t+1}$ \eqref{re_update} can be rewritten as
	\begin{equation} 
		\begin{aligned}
			&\balpha_{t+1} = \begin{bmatrix}	\balpha_t + q_{t+1}\bC_tK_t(\bxx_{t+1}) \\ q_{t+1}	\end{bmatrix}, \\
			&\bC_{t+1} = \\ &\begin{bmatrix}	\bC_t + r_{t+1}\bC_tK_t(\bxx_{t+1})K_t^T(\bxx_{t+1}) \bC_t^T  & r_{t+1}\bC_tK_t(\bxx_{t+1}) \\  r_{t+1}K_t^T(\bxx_{t+1}) \bC_t^T & r_{t+1}	\end{bmatrix} , 
		\end{aligned}
	\end{equation}
	The Gram function has $K_*(\bx_{t+1}) = [K_*(\bx_{t}), K_*(\bxx_{t+1})]$. Then, according to \eqref{RLGPR}, we obtain
	\begin{equation}
		\begin{aligned}
			&\bmu_{i,t+1}^{[L]} - \bmu_{i,t}^{[L]} = \\ 
			&[K_*(\bx_{t}), K_*(\bxx_{t+1})] \begin{bmatrix}	\balpha_t + q_{t+1}\bC_tK_t(\bxx_{t+1}) \\ q_{t+1}	\end{bmatrix} -K_*(\bx_t)\balpha_t \\
			&=q_{t+1} \Big(K_*(\bx_t)\bC_tK_t(\bxx_{t+1}) + K_*(\bxx_{t+1}) \Big).
		\end{aligned}
	\end{equation}
	Similarly,
	\begin{equation}
	\hat{\bSigma}_{i,t+1}^{[L]} - \hat{\bSigma}_{i,t}^{[L]} = r_{t+1} \Big(K_*(\bx_t)\bC_tK_t(\bxx_{t+1}) + K_*(\bxx_{t+1}) \Big)^2.
	\end{equation}
	Then, $\Delta\br_{i,t+1}^{[1]}$ \eqref{dr_1} can be rewritten as
	\begin{equation}
		\begin{aligned}
			&\Delta\br_{i,t+1}^{[1]} = \frac{\bmu_{i,t+1}^{[L]} \cdot \hat{\bSigma}_{i,t}^{[L]} - \bmu_{i,t}^{[L]} \cdot \hat{\bSigma}_{i,t+1}^{[L]}}{\hat{\bSigma}_{i,t+1}^{[L]} \cdot \hat{\bSigma}_{i,t}^{[L]}} \\
			&=\frac{ \big(\bmu_{i,t+1}^{[L]} - \bmu_{i,t}^{[L]} \big) \cdot \hat{\bSigma}_{i,t}^{[L]} - \bmu_{i,t}^{[L]} \cdot \big( \hat{\bSigma}_{i,t+1}^{[L]} - \hat{\bSigma}_{i,t}^{[L]} \big) }{\hat{\bSigma}_{i,t+1}^{[L]} \cdot \hat{\bSigma}_{i,t}^{[L]}}. \\
		\end{aligned}
	\end{equation}
	Since the only different point between time $t$ and $t+1$ is the $\bxx_{t+1}$, we have
	\begin{equation}
		\begin{aligned}
			&\mathop{\max}_t \left\lvert \bmu_{i,t+1} - \bmu_{i,t} \right\rvert = \frac{\sigma_f^2}{\sigma_e^2+\sigma_f^2} \left\lvert y_{t+1} \right\rvert, \\
			&\mathop{\max}_t \left\lvert \bSigma_{i,t} - \bSigma_{i,t+1} \right\rvert = \frac{\sigma_f^4}{\sigma_e^2+\sigma_f^2}.
		\end{aligned}
	\end{equation}
	which is obtained when $\bxx_{t+1} = \bxx^*$ and $\Vert \bxx_k - \bxx^* \Vert \rightarrow \infty,~k = 1,...,t$. Suppose that the Assumption~\ref{Sup:obs_bound} holds, we derive
	\begin{equation}
		\begin{aligned}
			\vert \Delta\br_{i,t+1}^{[1]} \vert  \leq \frac{\bar{y} \sigma_f^2} {\sigma_n^2(\sigma_e^2+\sigma_f^2)} + \frac{\bar{\mu} \sigma_f^4}{\sigma_n^2(\sigma_n^2+\sigma_f^2)(\sigma_e^2+\sigma_f^2)} = \delta_1.
		\end{aligned}
	\end{equation}
	Similarly, $\Delta\br_{i,t+1}^{[2]}$ is bounded as
	\begin{equation}
		\begin{aligned}
			\Delta\br_{i,t+1}^{[2]} = \frac{\hat{\bSigma}_{i,t}^{[L]} - \hat{\bSigma}_{i,t+1}^{[L]} }{\hat{\bSigma}_{i,t+1}^{[L]} \cdot \hat{\bSigma}_{i,t}^{[L]} } \leq \frac{ \sigma_f^4}{\sigma_n^2(\sigma_n^2+\sigma_f^2)(\sigma_e^2+\sigma_f^2)} = \delta_2.
		\end{aligned}
	\end{equation}
	Define $\delta_{\bxi_1},~\delta_{\bxi_2}$ as the consensus errors, i.e., 
	\begin{equation}
		\begin{aligned}
			&\bxi_{i,t+1}^{[1]} = \frac{1}{p} \sum_{i=1}^p \bmu_{i,t}^{[L]} \cdot (\hat{\bSigma}_{i,t}^{[L]})^{-1} + \delta_{\bxi_1}, \\
			&\bxi_{i,t+1}^{[2]} = \frac{1}{p} \sum_{i=1}^p (\hat{\bSigma}_{i,t}^{[L]})^{-1} + \delta_{\bxi_2},
		\end{aligned}
	\end{equation}
	According to the convergence of the FODAC algorithms (\textit{Theorem 3.1 in \cite{zhu_discrete-time_2010}}), suppose the Assumptions~\ref{Sup:PSC},~\ref{Sup:Non-degeneracy} hold, the consensus error $\delta_{\bxi_1},~\delta_{\bxi_2}$ is bounded as,
	\begin{equation}
		\begin{aligned}
			\lvert \delta_{\bxi_1} \rvert &
			 \leq \frac{4(pB-1)\delta_1}{\varphi^{0.5p(p+1)B-1}}, \\
			\lvert \delta_{\bxi_2} \rvert &\leq \frac{4(pB-1)\delta_2}{\varphi^{0.5p(p+1)B-1}}
		\end{aligned}
	\end{equation}
	when time $t \rightarrow \infty$.
	
Denote $\eta = \frac{4(pB-1)}{\varphi^{0.5p(p+1)B-1}}$,~ 
$\zeta_{\bmu} = \frac{1}{p} \sum_{i=1}^p \bmu_{i,t}^{[L]} \cdot (\hat{\bSigma}_{i,t}^{[L]})^{-1}$, 
$\zeta_{\bSigma} = \frac{1}{p} \sum_{i=1}^p (\hat{\bSigma}_{i,t}^{[L]})^{-1}$. 
Let $\lvert \delta_{\bxi_2} \rvert \leq \eta \delta_2 \leq \frac{1}{\sigma_n^2+\sigma_f^2} \leq \zeta_{\bSigma}$, i.e., the $\sigma_n^2$ 
satisfies $\sigma_n^2 \geq \frac{\eta \sigma_f^4}{\sigma_e^2+\sigma_f^2}$ such that we have $\zeta_{\bSigma} + \delta_{\bxi_2} > 0$ for avoiding singularity.
The difference between distributed the aggregation mean $\bmu_{i,t}^{[D]}$ and the PoE results \eqref{agg_result} can be written as,
	\begin{equation}
		\begin{aligned}
			&\lvert \bmu_{i,t}^{[D]} - \tilde{\bmu} \rvert \leq \left\lvert \frac{\zeta_{\bmu} + \delta_{\bxi_1}}{\zeta_{\bSigma} + \delta_{\bxi_2}} - \frac{\zeta_{\bmu}}{\zeta_{\bSigma}} \right\rvert \leq \left\lvert \frac{ \zeta_{\bSigma} \delta_{\bxi_1} - \zeta_{\bmu}\delta_{\bxi_2}}{ \zeta_{\bSigma} \left(\zeta_{\bSigma} + \delta_{\bxi_2}\right)  } \right\rvert \\
			& \leq  \left\lvert \frac{\delta_{\bxi_1}}{\zeta_{\bSigma} +\delta_{\bxi_2}} - \tilde{\bmu}(1-\frac{\zeta_{\bSigma}}{\zeta_{\bSigma} + \delta_{\bxi_2}})  \right\rvert \\
			& \leq \frac{\eta \delta_1 (\sigma_n^2 + \sigma_f^2)}{1 + \eta \delta_2 (\sigma_n^2 + \sigma_f^2)} + \left\lvert \frac{ \eta \delta_2 (\sigma_n^2 + \sigma_f^2)}{1 - \eta \delta_2 (\sigma_n^2 + \sigma_f^2)} \right\rvert \cdot \tilde{\bmu} 
		\end{aligned}
	\end{equation}
	Furthermore, if there is a constant $h>0$, $\Delta\br_{i,t+1} = \mathbf{0}$ for any time $t>h$ holds, 
	according to the Corollary 3.1 in \cite{zhu_discrete-time_2010}, $\delta_{\bxi_1}$ and $\delta_{\bxi_2}$ will converge to zero. Then, we have
	$\bmu_{i,t}^{[D]}(\bxx^*) \rightarrow \tilde{\bmu}(\bxx^*) ,~\forall i \in V,~ \bxx^* \in \mX$ as time $t \rightarrow \infty$.



\bibliography{IEEEabrv, sample}
\bibliographystyle{IEEEtran}

\end{document}